\documentclass{OurArticle}
\usepackage[utf8]{inputenc}
\usepackage{float}
\usepackage{soul}
\usepackage{placeins} 
\usepackage{tikz}
\usetikzlibrary{arrows.meta,positioning,fit,calc}
\usepackage[font=small]{caption} 

\title{GGMPs: Generalized Gaussian Mixture Processes}

\author[1,2,3,4]{Vardaan Tekriwal\thanks{\texttt{vtekriwal@berkeley.edu}}}
\author[5]{Mark D. Risser}
\author[1,6]{Hengrui Luo}
\author[1]{Marcus M. Noack}

\affil[1]{Applied Math \& Computational Research Division, Lawrence Berkeley National Laboratory, Berkeley, CA}
\affil[2]{Department of Electrical Engineering \& Computer Sciences, University of California, Berkeley, CA}
\affil[3]{Engineering Mathematics \& Statistics Division, University of California, Berkeley, CA}
\affil[4]{Department of Industrial Engineering \& Operations Research, University of California, Berkeley, CA}
\affil[5]{Climate \& Ecosystem Sciences Division, Lawrence Berkeley National Laboratory, Berkeley, CA}
\affil[6]{Department of Statistics, Rice University, Houston, TX}

\begin{document}
\date{}
\maketitle

\begin{abstract}
Conditional density estimation is complicated by multimodality, heteroscedasticity, and strong non-Gaussianity. Gaussian processes (GPs) provide a principled nonparametric framework with calibrated uncertainty, but standard GP regression is limited by its unimodal Gaussian predictive form. We introduce the \emph{Generalized Gaussian Mixture Process} (GGMP), a GP-based method for multimodal conditional density estimation in settings where each input may be associated with a complex output distribution rather than a single scalar response. GGMP combines local Gaussian mixture fitting, cross-input component alignment and per-component heteroscedastic GP training to produce a closed-form Gaussian mixture predictive density. The method is tractable, compatible with standard GP solvers and scalable methods, and avoids the exponentially large latent-assignment structure of naive multimodal GP formulations. Empirically, GGMPs improve distributional approximation on synthetic and real-world datasets with pronounced non-Gaussianity and multimodality.
\end{abstract}

\section{Introduction}
Gaussian processes (GPs) constitute a powerful framework for stochastic function approximation and uncertainty quantification, enabling principled Bayesian inference over functions while remaining nonparametric. Under a Gaussian likelihood, the latent function marginalizes analytically, yielding a closed-form log marginal likelihood for hyperparameter training and a Gaussian posterior predictive distribution at new inputs. The standard GP equations used throughout this paper are collected in 
Appendix~\ref{app:gp_background} for reference.

In this paper, we extend the GP framework
for settings where the conditional output distribution at each
input is  multimodal rather than a unimodal Gaussian. 
Specifically, we begin by describing a  multimodal
GP model in which each input is associated with a $K$-component
Gaussian mixture whose component means are governed by independent
latent functions with GP priors. In this case, we show that the resulting joint
likelihood is a mixture of $K^N$ Gaussian terms, which is computationally prohibitive even for modest datasets, motivating the need for a tractable alternative.

\begin{figure}
 \centering
 \includegraphics[width=\linewidth]{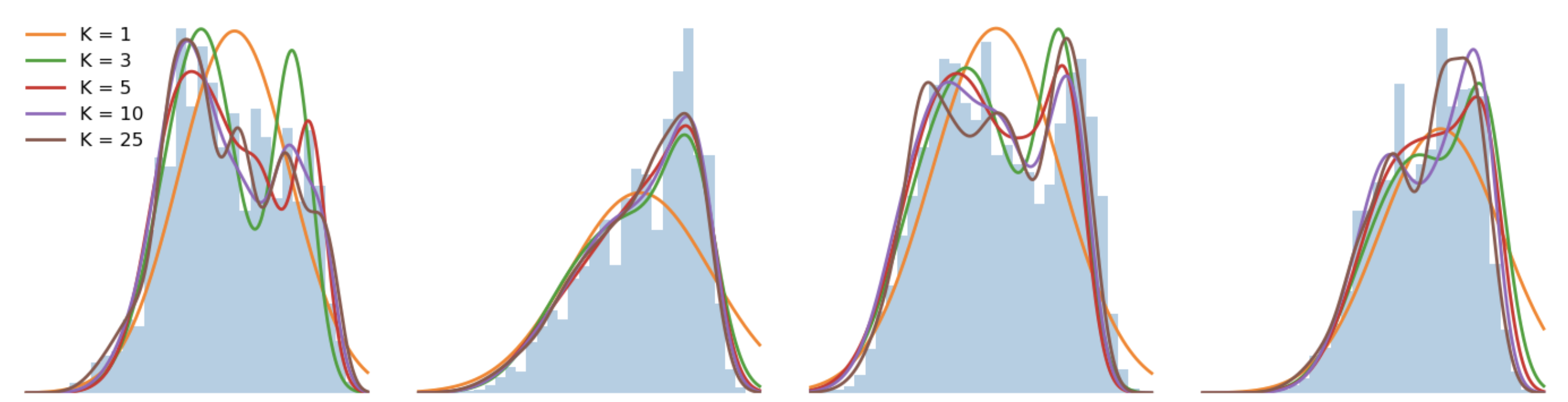}
 \caption{Generalized Gaussian Mixture Processes (GGMPs) capture multimodal conditional distributions while preserving closed-form inference. Many real-world processes exhibit conditioning-dependent  multimodality, asymmetry, and heteroscedasticity that standard GPs cannot represent. GGMPs address this by modeling $p(y \mid x)$ as a weighted mixture of GPs. Here, predictive densities at four held-out inputs are shown against empirical conditional distributions (blue histograms). With the number of components in the mixture $K{=}1$ (orange), GGMPs reduce to a standard heteroscedastic GP. Increasing $K$ progressively resolves multimodal structure, with the model remaining robust to overspecification as surplus components overlap with existing ones.}
 \label{fig:intro_predictive_fits}
\end{figure}

We then present the Generalized Gaussian Mixture Process (GGMP), an alternative that targets distribution-valued observations while retaining closed-form Gaussian conditioning, exemplified in Figure~\ref{fig:intro_predictive_fits}. The contributions of this work are as follows.

\begin{itemize}
 \item We introduce the GGMP as a computationally efficient, tractable alternative to the naive  multimodal GP model whose joint likelihood contains $K^N$ terms. The resulting three-stage pipeline---local Gaussian mixture fitting with cross-input component alignment, per-component heteroscedastic GP training (parallelizable across components), and weight optimization---reduces the effective complexity to $O(KN^3)$ while preserving a closed-form Gaussian-mixture predictive density. Here $N$ is typically much smaller than the total observation count, as the local mixture step summarizes repeated observations at each input. Beyond this inherent data compression, because each stage operates through standard GP primitives, the framework is compatible with any GP scaling method.

 \item We establish theoretical foundations for the framework. First, we propose a distributional maximum likelihood objective for distribution-valued data and show that it is equivalent, up to an additive constant, to minimizing the sum of forward KL divergences between observed and predicted densities. Second, we show that the GGMP family is a universal conditional density estimator: even under simplifying restrictions such as equal component weights and shared variance, it can approximate any continuous conditional density arbitrarily well.

 \item Alongside the theoretical guarantees, we demonstrate empirically that the modeling choices used to make the framework tractable—including local mixture fitting, cross-input alignment, plug-in heteroscedastic variances, and restricted weight parameterizations—are practically effective. Across synthetic and real-world datasets, and through targeted ablations of these simplifications, GGMPs recover multimodal conditional structure faithfully. 
\end{itemize}

\section{Related Work}

\subsection{Likelihood and prior transformations}

One classical route to flexible predictive distributions applies an invertible transformation to the response so
that a Gaussian likelihood is adequate in the transformed space (e.g.,
Box--Cox \cite{box1964analysis}, Sinh--Arcsinh \cite{jones2009sinh}); another
retains a latent GP with a non-Gaussian likelihood for counts, binary
outcomes, or heavy-tailed noise
\cite{williams2006gaussian,diggle1998model,shah2014student}. Generalized Gaussian process models provide a unifying exponential-family formulation for such settings \cite{chan2011generalized,shang2013approximate}. These approaches break
conjugacy, requiring variational \cite{titsias2009variational}, Laplace
\cite{williams1998bayesian}, EP \cite{minka2001expectation}, or MCMC
\cite{neal1997monte} approximations, and none can represent  multimodality
under monotone transforms. Warped GPs \cite{lazaro2012bayesian}, deep GPs
\cite{damianou2013deep,salimbeni2017doubly,noack2024unifying}, and normalizing-flow GPs
\cite{maronas2021transforming} transform the latent function for richer
predictive behavior but likewise sacrifice closed-form training. Heteroscedastic GP regression models input-dependent noise via a GP prior on
the log variance \cite{goldberg1998regression,lazaro2011variational}. A
simpler plug-in alternative treats locally estimated variances as fixed noise
levels, avoiding joint inference over a noise process at the cost of not
propagating variance uncertainty.

\subsection{GP mixtures and experts}

A large literature combines multiple GP components within a single predictive framework. Mixture-of-experts models
\cite{jacobs1991adaptive,jordan1994hierarchical,yuksel2012twenty,masoudnia2014mixture},
the Bayesian committee machine \cite{tresp2000bcm}, non-stationary GPs \citep{risser2015regression,risser2025compactly}, infinite mixtures of GP
experts \cite{rasmussen2002infinite}, treed GPs \cite{gramacy2008treed,luo2022sparse,luo2023sharded,luo2024hybrid}, and
sparse gated experts \cite{nguyen2014fast} all combine GP components, primarily to improve scalability, capture nonstationarity, or partition the input space \citep{noack2025gp2scale,luo2024non}. In these models, different GP components are typically activated in different regions of the domain or under latent assignment mechanisms, thus  multimodality at a fixed input arises indirectly. They therefore differ from our setting, where the objective is to model a multimodal conditional distribution at a single input directly and in closed form. OMGP \cite{lazaro2011overlapping} is closer in spirit, as it targets overlapping latent functions and can produce  multimodal predictions explicitly. However, it relies on variational inference with repeated GP factorizations inside the optimization loop, making it computationally prohibitive even for modest datasets.

\subsection{Conditional density estimation and distribution regression}

Classical conditional density estimation includes conditional KDE
\cite{silverman1986density,scott2015multivariate}. GP-based variants include GP-CDE \cite{dutordoir2018gaussian} and logistic GP models
\cite{lenk1988logistic,tokdar2007logistic,adams2009gpds}, which require
approximate inference \cite{rothfuss2019conditional}. Neural approaches include MDNs \cite{bishop1994mixture}, which parameterize a Gaussian mixture via a network trained on $(x,y)$ pairs, and conditional normalizing flows
\cite{winkler2019learning,trippe2018conditional}, which learn invertible
transformations of base densities. Relative to the GP-based approach proposed here, these neural methods require iterative optimization without closed-form conditioning, derive uncertainty from learned representations rather than a posterior, and lack an explicit smoothness
prior---a disadvantage when training inputs are scarce
(Section~\ref{sec:additive_manufacturing}). Distribution regression via kernel
mean embeddings \cite{muandet2017kernel,szabo2016learning} predicts embedded
summaries rather than densities, and reconstruction is generally
ill-posed \cite{poczos2013distribution,law2018bayesian}. When each input is
associated with samples or an empirical distribution rather than a single
observation, neither standard GP regression nor MDNs directly apply without
reducing the data to summaries or treating replicates as independent
pairs.

\section{Approach}

\subsection{ multimodal GP Model and Intractability}
\label{sec:intractability}
We model  multimodality by representing the conditional output at each input as a finite Gaussian mixture whose component means are governed by $K$ latent functions.
Denote the inputs by $\{x_n\}_{n=1}^N$.
For each component $k\in\{1,\ldots,K\}$ we introduce an independent latent function $f_k(x)$ with GP prior
\begin{equation}
 f_k(x) \sim \mathcal{GP}\,\big(m_k(x),\kappa_k(x,x')\big),
 \qquad k=1,\ldots,K,
 \label{eq:prior_moe_gp}
\end{equation}
and define $g_{nk} := f_k(x_n)$.
At input $x_n$, the conditional output density is modeled as
\begin{equation}
 \ell(y\mid x_n, \mathbf{g}_n)
 = \sum_{k=1}^{K} w_{nk}\,\mathcal N\,\big(y\mid g_{nk},\sigma^2_{nk}\big),
 \qquad \sum_{k=1}^{K} w_{nk}=1,\; w_{nk}\ge 0 \quad \forall\, n,\; k,
 \label{eq:per_input_mixture}
\end{equation}
where $\mathbf{g}_n:=(g_{n1},\ldots,g_{nK})^\top$ and the weights may depend on the input ($w_{nk}=w_k(x_n)$) to capture changing mixture proportions.

Assuming conditional independence across inputs given the latent functions, the dataset likelihood factorizes as a product of per-input mixtures. Let
$G:=\{g_{nk}\}_{n=1,k=1}^{N,K}$ denote the collection of latent component means at the training inputs. Then
\begin{align}
p(\mathbf{y}\mid G)
=\prod_{n=1}^{N} \ell(y_n\mid x_n, \mathbf{g}_n)
=\prod_{n=1}^{N}\sum_{k=1}^{K} w_{nk}\,
\mathcal N\!\big(y_n\mid g_{nk},\sigma^2_{nk}\big),
\label{eq:prod_of_sums_point}
\end{align}
which is a valid density when $\sum_{k=1}^{K} w_{nk}=1$ for all $n$.

Expanding the product of sums yields a joint Gaussian mixture indexed by an assignment
$r=(r_1,\ldots,r_{N})\in\{1,\ldots,K\}^{N}$:
\begin{align}
p(\mathbf{y}\mid G)
&=\sum_{r\in\{1,\ldots,K\}^{N}}
\prod_{n=1}^{N}
w_{n r_n}\,
\frac{1}{\sqrt{2\pi}\sigma_{n r_n}}
\exp\!\left[-\frac{1}{2\sigma_{n r_n}^2}(y_n-g_{n r_n})^2\right]
\nonumber\\
&=\sum_{r\in\{1,\ldots,K\}^{N}}
\frac{\prod_{n=1}^{N} w_{n r_n}}
{\sqrt{(2\pi)^{N}|\mathbf V_r|}}
\exp\!\left[-\frac{1}{2}(\mathbf y-\mathbf g_r)^\top \mathbf V_r^{-1}(\mathbf y-\mathbf g_r)\right],
\label{eq:EnGPLikelihoodDerivation}
\end{align}
where $\mathbf V_r:=\mathrm{diag}(\sigma^2_{1r_1},\ldots,\sigma^2_{Nr_{N}})$ and
$\mathbf g_r:=(g_{1r_1},\ldots,g_{Nr_{N}})^\top$.
Because the assignment vector $r$ ranges over all of $\{1,\ldots,K\}^N$,
the sum in \eqref{eq:EnGPLikelihoodDerivation} contains exactly $K^N$ terms,
each being an $N$-dimensional Gaussian with diagonal covariance.
Even for moderate $K$ and $N$, direct evaluation or marginalization of the joint likelihood becomes intractable. 

Despite this intractability, the per-input mixture form
\eqref{eq:per_input_mixture} is itself highly expressive. Even under
strong assumptions of equal weights and a shared variance---stronger than we later adopt (Section~\ref{sec:implementation})---the model class can approximate any continuous conditional density to arbitrary accuracy as the number of components grows. We formalize this next.

\begin{proposition}[Universal conditional density approximation]
\label{prop:universality}
Let $\mathcal{X} \subset \mathbb{R}^d$ be compact and let $p^*(y \mid x)$ be a conditional density on $\mathbb{R}$ such that $(x,y) \mapsto p^*(y \mid x)$ is jointly continuous on $\mathcal{X} \times \mathbb{R}$ and the family $\{p^*(\cdot \mid x)\}_{x \in \mathcal{X}}$ is uniformly tight. Then for every $\epsilon>0$ there exist an integer $K$, continuous functions $\mu_1,\ldots,\mu_K:\mathcal{X}\to\mathbb{R}$, and a scalar $\sigma^2>0$ such that, where $D_{\mathrm{TV}}$ denotes the total variation distance,

$$\sup_{x \in \mathcal{X}} \, D_{\mathrm{TV}}\!\left(p^*(\cdot \mid x),\; \frac{1}{K}\sum_{k=1}^{K} \mathcal{N}(\cdot \mid \mu_k(x), \sigma^2)\right) < \epsilon.$$
\end{proposition}

The proof, given in Appendix~\ref{app:universality}, constructs the approximation via a partition-of-unity argument: at each input $x$, the target density is approximated by a location-scale Gaussian mixture whose means are then extended to continuous functions over the compact domain. While this property is shared by several model classes in the literature that typically require approximate inference, the proposition does not, per se, imply that equal weights or shared variance are statistically optimal at finite $K$ and sample size; it shows that these design choices are parsimonious constraints that preserve representational power while reducing degrees of freedom. Further, its significance is that the GGMP achieves universality \emph{while} retaining closed-form GP conditioning in every component.

This observation motivates the GGMP: rather than approximating the intractable joint model \eqref{eq:EnGPLikelihoodDerivation}, we adopt a different, decoupled model that retains the per-input mixture form \eqref{eq:per_input_mixture} but replaces joint inference over all $K^N$ assignment configurations with $K$ independent GP regressions on deterministically aligned component targets. The resulting framework is therefore not a variational or asymptotic approximation to the joint model; it is a separate model class that retains the same universal approximation property as the restricted per-input mixture family described above.

\subsection{From point observations to distribution-valued data}
\label{sec:point_to_dist}
In many applications---such as stochastic simulation, repeated physical
experiments, or ensemble forecasts---each input $x_n$ is naturally
associated with multiple output observations rather than a single scalar response.
Equations \eqref{eq:prod_of_sums_point}--\eqref{eq:EnGPLikelihoodDerivation}
describe the likelihood of a single scalar observation $y_n$ at each input $x_n$.

However, in our setting, the datum at $x_n$ is a distribution $p_n(y)$, typically represented by $T_n$ samples $\{Y_{nt}\}_{t=1}^{T_n}$ collected at the same input or by a discretized density; for simplicity, we write the formulas below for the balanced case $T_n \equiv T$. The per-input log-likelihood becomes an expectation under $p_n$, approximated in practice by the sample average $\frac{1}{T}\sum_{t=1}^{T}\log p(Y_{nt}\mid\theta)$; for this average to converge as $T\to\infty$, we require
$\mathbb{E}_{p_n}\!\big[\big|\log\!\big(\sum_{k=1}^{K} w_{nk}\,
\mathcal{N}(Y\mid g_{nk},\sigma^2_{nk})\big)\big|\big] < \infty$
for each $n$.

Since a Gaussian mixture density is strictly positive on $\mathbb{R}$, the integrand is bounded above; the condition reduces to $\int p_n(y)\,\log^{-}\!\big(\sum_k w_{nk}\,\mathcal{N}(y\mid g_{nk},\sigma^2_{nk})\big)\,dy < \infty$. Because the density of the Gaussian mixture decays as $\exp(-cy^2)$ in the tails for some $c > 0$, we have $\log^{-}\!\big(\sum_k w_{nk}\,\mathcal{N}(y\mid g_{nk},\sigma^2_{nk})\big) = O(y^2)$ as $|y| \to \infty$, so the integrability condition holds whenever $p_n$ has a finite second moment, i.e. $\int p_n(y)\, y^2\, dy < \infty$. This is substantially weaker than sub-Gaussianity. In the empirical setting where $p_n$ is represented by finitely many samples, the condition is satisfied automatically.

Under the per-input mixture model \eqref{eq:per_input_mixture}, the sample log-likelihood is
\begin{equation}
\mathcal L_T
=
\sum_{n=1}^{N}\sum_{t=1}^T
\log\!\Bigg(\sum_{k=1}^{K} w_{nk}\,
\mathcal N\!\big(Y_{nt}\mid g_{nk},\sigma^2_{nk}\big)\Bigg).
\label{eq:sample_ll}
\end{equation}
Dividing by $T$ and taking $T\to\infty$ yields the population objective
(by the strong law of large numbers, applied to each summand $n$ separately):

\begin{equation}
\frac{1}{T}\mathcal L_T
\;\xrightarrow{\;\mathrm{a.s.}\;}\;
\sum_{n=1}^{N}
\int p_n(y)\,
\log\!\Bigg(\sum_{k=1}^{K} w_{nk}\,
\mathcal N\!\big(y\mid g_{nk},\sigma^2_{nk}\big)\Bigg)\,dy.
\label{eq:dist_ll_full_form}
\end{equation}
When $p_n$ is represented by samples, the population integral is estimated
by its empirical counterpart. Writing $\hat{P}_n = \frac{1}{T}\sum_{t=1}^T \delta_{Y_{nt}}$
for the empirical measure, the per-input contribution to \eqref{eq:dist_ll_full_form} becomes
\begin{equation}
\int \log\!\Big(\sum_{k=1}^{K} w_{nk}\,\mathcal N(y\mid g_{nk},\sigma^2_{nk})\Big)\,d\hat P_n(y)
=
\frac{1}{T}\sum_{t=1}^{T}\log\!\Big(\sum_{k=1}^{K} w_{nk}\,
\mathcal N(Y_{nt}\mid g_{nk},\sigma^2_{nk})\Big),
\end{equation}
which is the standard sample average log-likelihood at input $x_n$.

The objective \eqref{eq:dist_ll_full_form} is the natural analog of maximum likelihood when each input is associated with an observed distribution rather than a single draw.
The equivalence between maximizing an expected log-likelihood under a reference distribution and minimizing the forward KL divergence is classical; see, e.g., \citet{cover2006elements,burnham2002model}.
We restate it here in the specific form needed for our distribution-valued setting.

\begin{lemma}[Distributional MLE = KL minimization]
\label{distribution_MLE}
Let $(\mathcal Y,\mathcal A,\mu)$ be a measurable space with reference measure $\mu$
(counting measure if $\mathcal Y$ is discrete, Lebesgue measure if $\mathcal Y$ is continuous).
Let $\{(x_n,p_n)\}_{n=1}^N$ be data, where each $p_n:\mathcal Y\to[0,\infty)$ is a probability
density/mass function with respect to $\mu$. Let $q_\theta(\cdot\mid x_n)$ be a model-predicted density/mass function with respect to $\mu$.
Assume all integrals below are finite.

Define the distributional log-likelihood
\[
\mathcal L(\theta)
:=\sum_{n=1}^N \int_{\mathcal Y} p_n(y)\,\log q_\theta(y\mid x_n)\,d\mu(y).
\]
Then
\[
\arg\max_\theta \mathcal L(\theta)
=
\arg\min_\theta \sum_{n=1}^N D_{\mathrm{KL}}\!\big(p_n \,\|\, q_\theta(\cdot\mid x_n)\big),
\]
and for all $\theta$,
\[
\mathcal L(\theta)
=
-\sum_{n=1}^N H(p_n)
-\sum_{n=1}^N D_{\mathrm{KL}}\!\big(p_n \,\|\, q_\theta(\cdot\mid x_n)\big),
\]
where $H(p_n):=-\int_{\mathcal Y} p_n(y)\log p_n(y)\,d\mu(y)$
is the Shannon entropy (discrete case) or differential entropy (continuous case).
\end{lemma}

The derivation follows by expanding the KL divergence and separating the $\theta$-independent entropy term; we include it in Appendix~\ref{app:proof_prop1} for completeness.

Lemma \ref{distribution_MLE} serves two purposes.
First, it justifies \eqref{eq:dist_ll_full_form} as the likelihood-based objective associated with fitting predictive densities to distribution-valued observations.
Second, it clarifies the learning criterion: we fit $q_\theta(\cdot\mid x_n)$ to be close to $p_n$ in forward KL, aggregated over inputs.
When $p_n$ is available only through samples $Y_{n1},\ldots,Y_{nT_n}$, the integral is approximated by Monte Carlo,
$\int p_n \log q_\theta \approx \frac{1}{T_n}\sum_{t=1}^{T_n} \log q_\theta(Y_{nt}\mid x_n)$,
recovering the usual sample-average log-likelihood estimator.

\subsection{Implementation}
\label{sec:implementation}

The distribution-valued likelihood in \eqref{eq:dist_ll_full_form} is a
tractable and principled objective for fitting predictive densities to
distribution-valued observations. The challenge is instead model
construction: the naive multimodal GP formulation
\eqref{eq:EnGPLikelihoodDerivation} is exponentially intractable, and even
the per-input mixture representation leaves component labels ambiguous
across inputs. A practical global model therefore requires both a
consistent cross-input component correspondence and a choice of whether
weights are shared or input-dependent.

We address these challenges with a three-step pipeline. First, we fit a local Gaussian mixture at each input and align components to establish a consistent labeling. Second, we train one GP per component. Third, we combine the resulting component predictive densities into a global mixture and then weight under the distributional log-likelihood from Lemma~\ref{distribution_MLE}. We consider three weight parameterizations---equal, shared, and input-dependent---and Proposition~\ref{prop:universality} provides theoretical support for these design choices. In particular, because universality requires only that the component means vary with the input, the model can retain full representational capacity even when weights and variances are restricted to be input-independent, acknowledging that such restrictions need not be optimal at finite $K$ or $N$. We assess the practical adequacy of these simplifications in Section~\ref{sec:additive_manufacturing} for plug-in variances and in Section~\ref{sec:weight_effect} for weight parameterizations, and discuss extensions in Section~\ref{sec:limitations}. We now describe each step in turn.

\subsubsection{Local Gaussian Mixture Fitting and Component Alignment}
\label{sec:fitting}

The number of components $K$ is a user-specified structural choice. A detailed analysis of the failure modes is provided in Appendix~\ref{app:K_sensitivity}. In short, under-specification merges modes, inflating within-component variance and placing mass between true modes; over-specification introduces redundant components. In the infinite-data limit, if redundant components converge to identical parameters, the predictive density depends only on their combined weight, so the shared-weight objective becomes flat along redistribution directions. With finite samples, the reduced effective sample size per component can still degrade estimation. We justify the restrictions empirically in Section~\ref{sec:experiments}, and discuss practical selection strategies in Section~\ref{sec:discussion}.

For each input $x_n$, we assume the observed density $p_n$ is represented either by samples $\{Y_{nt}\}_{t=1}^T$
or by a discretized histogram/KDE on a grid.
We fit a $K$-component Gaussian mixture to the corresponding samples at $x_n$, producing
local component parameters
\[
\big\{(\hat\omega_{n\kappa},\,\hat m_{n\kappa},\,s^2_{n\kappa})\big\}_{\kappa=1}^K,
\]
where $\hat\omega_{n\kappa}$ is the fitted local mixture weight, $\hat m_{n\kappa}$ is the fitted component mean, and $s^2_{n\kappa}$ is the within-component variance.

Because a Gaussian mixture is invariant under permutation of its components \citep{kunkel2020anchored},
the local index $\kappa$ returned by the EM algorithm at input $x_n$
carries no cross-input meaning.
That said, to train one GP per mixture component, we require a consistent correspondence
of component labels across the $N$ inputs.

Establishing consistent component labels across mixture fits is a
well-studied problem in the Bayesian mixture literature, where it arises
as the \emph{label-switching} problem: because the mixture likelihood is
invariant under permutation of its components, MCMC samplers explore
symmetric posterior modes and relabel post-hoc to produce
interpretable summaries
\citep{stephens2000dealing,jasra2005markov,papastamoulis2010artificial}.
Common remedies include ordering constraints on component means,
loss-based relabeling against a reference configuration, and pivotal
reordering. Our alignment step is a frequentist counterpart: rather than
relabeling posterior samples within a single mixture, we relabel point
estimates across $N$ independent fits to obtain coherent GP training
targets. The key difference is that our setting exploits spatial
continuity of the latent mean functions across the input domain,
motivating optimal-transport-based costs that measure geometric
dissimilarity between components as distributions rather than purely
label-based matching.

Our default alignment is simple: at each input $x_n$, we sort the fitted components
in ascending order of their means and assign global label $k$ to the component
with the $k$-th smallest mean. Formally, let $\pi_n : \{1,\ldots,K\} \to \{1,\ldots,K\}$
be a permutation satisfying
\begin{equation}
 \hat{m}_{n,\pi_n(1)} \;\le\; \hat{m}_{n,\pi_n(2)} \;\le\; \cdots \;\le\; \hat{m}_{n,\pi_n(K)},
 \label{eq:sort_alignment}
\end{equation}
and relabel $\hat{m}_{nk} := \hat{m}_{n,\pi_n(k)}$,\;
$s^2_{nk} := s^2_{n,\pi_n(k)}$ for all $k$.

This choice is motivated by the theory of optimal transport in one dimension:
for univariate distributions, the Wasserstein-$p$ optimal coupling for any $p \ge 1$
is the monotone (quantile) coupling, which preserves rank ordering
\citep{villani2003topics}.
When the $K$ component means at each input are viewed as a discrete point configuration
on $\mathbb{R}$, sorting produces the assignment that minimizes the total squared displacement
$\sum_{k=1}^K (\hat{m}_{nk} - \hat{m}_{n'k})^2$ between any pair of inputs,
provided the rank order of the true component mean functions is preserved across inputs.
This condition holds exactly when $f_1(x) < f_2(x) < \cdots < f_K(x)$ for all $x$,
and approximately when crossings are infrequent relative to inter-component separation.

If two component mean functions do cross---i.e.,
$f_j(x') < f_k(x')$ but $f_j(x'') > f_k(x'')$ for some inputs $x' < x''$---then
sorting by local means swaps the global labels of these components on either side
of the crossing, creating a discontinuity in the GP training targets.
In such cases, sorting can be replaced by a sequential matching procedure:
fix an ordering at a reference input $x_1$ and, for each subsequent input $x_n$,
solve the linear assignment problem
\begin{equation}
 \pi_n = \arg\min_{\pi \in S_K} \sum_{k=1}^{K}
 c\,\big(\mathcal{N}(\hat{m}_{n-1,k},\, s^2_{n-1,k}),\;
  \mathcal{N}(\hat{m}_{n,\pi(k)},\, s^2_{n,\pi(k)})\big),
 \label{eq:hungarian_alignment}
\end{equation}
where $S_K$ is the symmetric group on $K$ elements and
$c(\cdot,\cdot)$ is a dissimilarity measure.
A natural choice is the squared $\mathcal{W}_2$ Wasserstein distance,
which for two multivariate Gaussians takes the form
\begin{equation}
 c\big(\mathcal{N}(\boldsymbol{\mu}_1, \boldsymbol{\Sigma}_1),\,
  \mathcal{N}(\boldsymbol{\mu}_2, \boldsymbol{\Sigma}_2)\big)
 = \|\boldsymbol{\mu}_1 - \boldsymbol{\mu}_2\|^2
 + \mathrm{tr}\,\big(\boldsymbol{\Sigma}_1
 + \boldsymbol{\Sigma}_2
 - 2(\boldsymbol{\Sigma}_1^{1/2}\,\boldsymbol{\Sigma}_2\,
  \boldsymbol{\Sigma}_1^{1/2})^{1/2}\big),
 \label{eq:W2_cost}
\end{equation}
though the costs such as the squared Hellinger distance may also be used.
For small $K$, \eqref{eq:hungarian_alignment} can be solved in $O(K^3)$ time \citep{kuhn1955hungarian},
making the total alignment cost $O(NK^3)$.

The sorting alignment \eqref{eq:sort_alignment} exploits the total order on $\mathbb{R}$ and is specific to univariate outputs. For multivariate outputs $y \in \mathbb{R}^p$, no canonical ordering of the component means exists. For multivariate outputs, we order inputs lexicographically and apply the sequential matching procedure \eqref{eq:hungarian_alignment} with the multivariate cost \eqref{eq:W2_cost}; this provides a dimension-agnostic alternative that requires no changes to the GP training or weight optimization stages. Experiments in Section~\ref{sec:experiments} demonstrate both cases $p = 1$ and $p > 1$. In the multivariate case, the local GMM fit produces mean vectors $\hat{\boldsymbol{m}}_{n\kappa} \in \mathbb{R}^p$ and covariance matrices $\mathbf{S}_{n\kappa} \in \mathbb{R}^{p \times p}$; the full multivariate predictive density is derived in Appendix~\ref{app:multivariate}.

\subsubsection{Heteroscedastic Component GP Training and Weight Optimization}
\label{sec:training}
Given the alignment permutations $\{\pi_n\}_{n=1}^N$,
we define the $k$-th GP training set as
\begin{equation}
 \mathcal{D}_k := \{(x_n, \hat{m}_{nk}, s^2_{nk})\}_{n=1}^{N}.
 \label{eq:Dk_def_impl}
\end{equation}
For each $k\in\{1,\ldots,K\}$ we fit $\mathrm{GP}_k$ on $\mathcal D_k$,
treating $s^2_{nk}$ as known heteroscedastic observation-noise variances.
Let $h:=\{h_k\}_{k=1}^K$ denote all GP hyperparameters.
The trained GP has a posterior predictive distribution for the latent component mean at $x_n$,
\begin{equation}
g_{nk}\mid(h_k,\mathcal{D}_k)\sim \mathcal{N}\,\big(\mu_{nk}(h_k),\,\nu_{nk}(h_k)\big).
\label{eq:gp_post_impl}
\end{equation}
We use the within-component variance as likelihood variance, i.e. we set $\sigma^2_{nk}:=s^2_{nk}$.\footnote{Note that $s^2_{nk}$ reflects the full within-component spread rather than the estimation uncertainty in $\hat m_{nk}$, which is typically smaller and scales like $s^2_{nk}/T_{nk}^{\mathrm{eff}}$, where $T_{nk}^{\mathrm{eff}} := T\,\hat\omega_{nk}$ and $\hat\omega_{nk}$ is the fitted local mixture weight of component $k$ at input $x_n$. Using $s^2_{nk}$ as the GP noise variance therefore treats the aligned component means as noisier observations than they would be under a pure mean-estimation view. An alternative would be to train the GP using $s^2_{nk}/T_{nk}^{\mathrm{eff}}$ and then add back the full $s^2_{nk}$ when forming the predictive density \eqref{eq:q_nk_density_impl}. When $T_{nk}^{\mathrm{eff}}$ is large, the two choices are similar; when $T_{nk}^{\mathrm{eff}}$ is small, the current choice can over-smooth the GP posterior mean. We examine the consequences in Sections~\ref{sec:additive_manufacturing} and~\ref{sec:limitations}.}

Marginalizing out $g_{nk}$ produces an explicit Gaussian component predictive density,
\begin{equation}
\begin{aligned}
q_{nk}(y;h_k)
&:= \int \mathcal{N}\,\big(y\mid g_{nk},\,s^2_{nk}\big)\,
 \mathcal{N}\,\big(g_{nk}\mid \mu_{nk}(h_k),\,\nu_{nk}(h_k)\big)\,dg_{nk} \\
&= \mathcal{N}\,\Big(y \,\Big|\, \mu_{nk}(h_k),\,\nu_{nk}(h_k)+s^2_{nk}\Big).
\end{aligned}
\label{eq:q_nk_density_impl}
\end{equation}
We introduce the simplex notation for the weight collection $W:=\{w_n\}_{n=1}^N$, where each $w_n\in\Delta^{K-1}$ (where $\Delta^{K-1} = \{w \in \mathbb{R}^K : w_{k} \geq 0,\, \sum_k w_{k} = 1\}$). We define the GGMP predictive density at $x_n$ as
\begin{equation}
q_n(y;h,W)=\sum_{k=1}^{K} w_{nk}\,q_{nk}(y;h_k).
\label{eq:ggmp_station_density_k_impl}
\end{equation}
Intuitively, $q_{nk}$ captures \emph{where} the $k$-th mode sits (and how uncertain it is) at input $x_n$,
while $w_n$ controls the mixing proportions of these modes at $x_n$.

Given observed densities $\{p_n\}$, we fit parameters by maximizing
\begin{equation}
\mathcal{L}_{\mathrm{dens}}(h,W)
:= \sum_{n=1}^{N}\int p_n(y)\,\log q_n(y;h,W)\,dy.
\label{eq:ggmp_dens_obj_impl}
\end{equation}
By Lemma~\ref{distribution_MLE}, maximizing \eqref{eq:ggmp_dens_obj_impl} is equivalent (up to the constant $-\sum_n H(p_n)$)
to minimizing $\sum_n D_{\mathrm{KL}}\!\big(p_n \,\|\, q_n(\cdot;h,W)\big)$, 
and in the empirical case
$p_n = \hat{P}_n = \frac{1}{T}\sum_{t=1}^T \delta_{Y_{nt}}$,
the objective \eqref{eq:ggmp_dens_obj_impl} reduces to the sample average log-likelihood:
\begin{equation}
\int \log q_n(y;h,W)\,d\hat P_n(y)
= \frac{1}{T}\sum_{t=1}^T \log q_n(Y_{nt};h,W). 
\end{equation}
When $p_n$ is represented on a one-dimensional grid $\{y_{n\ell}\}_{\ell=1}^{M_n}$ with quadrature weights
$\{\Delta y_{n\ell}\}_{\ell=1}^{M_n}$ and normalization $\sum_\ell p_n(y_{n\ell})\Delta y_{n\ell}=1$,
we approximate \eqref{eq:ggmp_dens_obj_impl} as
\begin{equation}
\mathcal{L}_{\mathrm{dens}}(h,W)
\approx
\sum_{n=1}^{N}\sum_{\ell=1}^{M_n}
p_n(y_{n\ell})\,\Delta y_{n\ell}\,\log q_n(y_{n\ell};h,W).
\label{eq:ggmp_dens_obj_quad_impl}
\end{equation}
To compute $\log q_n$ stably, define
\begin{equation}
\mathcal{B}^{(n)}_{k}(y)
:= \log(w_{nk})
+ \log \mathcal{N}\!\Big(y \,\Big|\, \mu_{nk}(h_k),\,\nu_{nk}(h_k)+s^2_{nk}\Big),
\label{eq:B_k_density_impl}
\end{equation}
so that
\begin{align}
\log q_n(y;h,W)
&=\log\sum_{k=1}^{K}\exp\!\big(\mathcal{B}^{(n)}_{k}(y)\big)\nonumber\\
&=\mathcal{B}^{(n)}_{k_0}(y)+
\log\!\Bigg(\sum_{k=1}^{K}\exp\!\big(\mathcal{B}^{(n)}_{k}(y)-\mathcal{B}^{(n)}_{k_0}(y)\big)\Bigg),
\label{eq:ggmp_logsumexp_k_impl}
\end{align}
where $k_0\in\arg\max_k \mathcal{B}^{(n)}_{k}(y)$.\footnote{The approximation \eqref{eq:ggmp_dens_obj_quad_impl} is a standard quadrature estimate of \eqref{eq:ggmp_dens_obj_impl} and inherits the convergence rate of the chosen rule. For non-uniform grids $\{y_{n\ell}\}$, the weights $\Delta y_{n\ell}$ must be proper quadrature weights (e.g.\ trapezoidal widths $\Delta y_{n\ell}=\tfrac{1}{2}(y_{n,\ell+1}-y_{n,\ell-1})$), not a fixed constant. Under mild regularity---namely, if $f_n(y):=p_n(y)\log q_n(y;h,W)$ is twice continuously differentiable on a bounded interval $[a_n,b_n]$---the trapezoidal rule gives $\left|\int_{a_n}^{b_n} f_n(y)\,dy-\sum_{\ell=1}^{M_n} f_n(y_{n\ell})\,\Delta y_{n\ell}\right|\le \frac{b_n-a_n}{12}\max_y |f_n''(y)|(\Delta y_n^{\max})^2$, where $\Delta y_n^{\max}:=\max_\ell \Delta y_{n\ell}$, yielding total error $O(N(\Delta y^{\max})^2)$ over all $N$ inputs. Under the shared-weight parameterization introduced below, each term in \eqref{eq:ggmp_dens_obj_quad_impl} has the form $\alpha_{n\ell}\log\!\big(\sum_k w_k q_{nk}(y_{n\ell})\big)$ with $\alpha_{n\ell}:=p_n(y_{n\ell})\Delta y_{n\ell}\ge 0$, so the discretized objective remains concave in $w$ on $\Delta^{K-1}$ and the optimization guarantees below are unchanged. In all experiments (Section~\ref{sec:experiments}), the grids are fine enough that this error is negligible.}

The distributional log-likelihood \eqref{eq:ggmp_dens_obj_impl} provides a principled training objective,
but optimizing it requires specifying a concrete computational procedure.
The parameters to be determined are the GP hyperparameters $h = \{h_k\}_{k=1}^K$
together with a choice of mixture weights. In principle, one could optimize
\eqref{eq:ggmp_dens_obj_impl} jointly over $(h,W)$, as elaborated in
Section~\ref{sec:limitations}; in practice, the decoupled procedure
described below performs well and may even be preferable, as the
per-component GP fit is a standard marginal-likelihood optimization,
avoiding the high-dimensional non-convex joint problem. This decoupling
also allows the weight strategy to be chosen independently.

We proceed with a two-stage procedure using the shared-weight scheme ($w_{nk} = w_k$), and defer the input-dependent weighting method to Section~\ref{sec:input_dependent_weights}\footnote{The equal-weight scheme requires no optimization, since once the number of components $K$ is specified, each weight is fixed at $1/K$.}.
In the first stage, each component GP is trained independently:
the hyperparameters $h_k$ are determined by maximizing the standard GP log marginal likelihood on the aligned dataset $\mathcal{D}_k$,
\begin{equation}
 h_k^* = \arg\max_{h_k} \log p\,\big(\hat{\mathbf{m}}_k \mid \mathbf{x}, h_k, \{s^2_{nk}\}_{n=1}^N\big),
 \qquad k = 1, \ldots, K,
 \label{eq:gp_mll}
\end{equation}
where $\hat{\mathbf{m}}_k := (\hat{m}_{1k}, \ldots, \hat{m}_{Nk})^\top$ is the vector of aligned
component means and the heteroscedastic noise variances $\{s^2_{nk}\}$ enter the likelihood
covariance as $\mathbf{K}_k + \mathrm{diag}(s^2_{1k}, \ldots, s^2_{Nk})$,
with $\mathbf{K}_k$ denoting the kernel matrix under $\kappa_k$.
Each such optimization is a standard single-output GP model selection problem
and can be solved as usual.

In the second stage, holding $h = h^*$ fixed, we optimize the shared weight vector $w \in \Delta^{K-1}$ by maximizing the distributional log-likelihood
\begin{equation}
 w^* = \arg\max_{w \in \Delta^{K-1}} \mathcal{L}_{\mathrm{dens}}(h^*, w),
 \label{eq:weight_opt}
\end{equation}
where $\mathcal{L}_{\mathrm{dens}}$ is defined in \eqref{eq:ggmp_dens_obj_impl}
(or its quadrature approximation \eqref{eq:ggmp_dens_obj_quad_impl}).
Under the shared-weight parameterization, this is a concave maximization problem over $\Delta^{K-1}$, since $\mathcal{L}_{\mathrm{dens}}$ is a sum of $\log(\cdot)$ composed with
an affine function of $w$ at each quadrature/sample point,
and the log of a nonnegative affine function is concave in the mixture weights. The global optimum can therefore be found reliably by projected gradient ascent,
exponentiated gradient updates, or any method that solves for concave programs on the simplex.

This two-stage design is fully parallelizable across the $K$ components
and reduces the weight optimization to a low-dimensional concave problem
on the simplex. The cost is that $(h^*, w^*)$ is not guaranteed to be a
stationary point of the joint objective $\mathcal{L}_{\mathrm{dens}}(h, w)$ since each GP is trained under its own marginal likelihood rather than the
full mixture objective. This decoupling is reasonable when the
local GMM fitting and alignment step produces stable, non-crossing component
tracks: each $\mathcal{D}_k$ then resembles a well-posed single-output
regression problem whose marginal-likelihood hyperparameters capture the
smoothness and length scale that the mixture objective requires. 

The overall training complexity is $O(KN^3)$---the same scaling as fitting
$K$ independent GPs---dominated by the Cholesky factorization of the
$N \times N$ kernel-plus-noise matrix in each component. Local mixture
fitting adds $O(NI_{\mathrm{EM}}TK)$ and the simplex weight optimization
adds $O(I_w NMK)$, both modest relative to GP training. At test time,
prediction costs $O(KN^2)$ per new input. In the multivariate case
($p > 1$), a separate scalar GP is trained for each output dimension
within each component, giving $Kp$ models and training/prediction costs
of $O(KpN^3)$ and $O(KpN^2)$, respectively. A deeper analysis is provided in Appendix~\ref{app:multivariate}.

\subsubsection{Prediction at new inputs}
\label{sec:prediction}

Given the trained hyperparameters $h^*$ and optimized weights $w^*$,
the model yields a closed-form predictive density at any new input $x_*$.
For each component $k$, the GP posterior predictive at $x_*$ gives
$g_k(x_*) \mid (h_k^*, \mathcal{D}_k) \sim \mathcal{N}\!\big(\mu_{*k}(h_k^*),\, \nu_{*k}(h_k^*)\big)$,
where $\mu_{*k}$ and $\nu_{*k}$ are the standard GP predictive mean and variance
conditioned on the aligned training data $\mathcal{D}_k$.
Marginalizing latent component means as in \eqref{eq:q_nk_density_impl},
the $k$-th predictive density is
\begin{equation}
 q_{*k}(y; h_k^*) = \mathcal{N}\!\Big(y \,\Big|\, \mu_{*k}(h_k^*),\; \nu_{*k}(h_k^*) + \bar{s}^2_k\Big),
 \label{eq:pred_component}
\end{equation}
where $\bar{s}^2_k := \frac{1}{N}\sum_{n=1}^{N} s^2_{nk}$ is the mean of the
training-set within-component variances for component $k$.
Although training uses heteroscedastic noise variances $s^2_{nk}$ at each input,
at a new location no local mixture fit is available. When the within-component variance varies substantially across the input domain,
a more refined approach would interpolate $s^2_{nk}$ as a function of $x$
(e.g., via a second GP or a parametric model) to obtain an input-dependent
prediction noise variance $\hat{s}^2_{*k}(x_*)$.
We use the simpler average in this work and explain its empirical adequacy in Section~\ref{sec:experiments}.
The full predictive density is then
\begin{equation}
 q(y \mid x_*;\, h^*, w^*) = \sum_{k=1}^{K} w_k^*\, q_{*k}(y; h_k^*).
 \label{eq:pred_new_input_inf}
\end{equation}
Prediction requires only standard GP inference at $x_*$ for each component
followed by mixture evaluation, and is therefore $O(KN^2)$ per test input.

\section{Experiments}
\label{sec:experiments}
We evaluate the GGMP under a reproducible, compute-budget--matched baseline set
focused on methods that scale to our regimes and have well-maintained
implementations. We compare against (i) a standard heteroscedastic GP
($K=1$, equivalent to $\mathrm{GGMP}_1$) as a GP-family unimodal baseline,
and (ii) Mixture Density Networks (MDNs) with matched component counts
$K \in \{1,3,5,10,25\}$. MDNs parameterize mixture means, variances, and
weights as outputs of a single neural network, providing a strong and
scalable neural conditional density estimator with the same Gaussian-mixture
predictive form. All MDNs use common default architectures and
hyperparameters from the MDN literature \cite{bishop1994mixture,rothfuss2019conditional}
(details in Appendix~\ref{app:mdn_details}) and are trained by maximizing
the same distributional log-likelihood used for the GGMP weight optimization.

We do not include OMGP or GP-CDE baselines due to feasibility and
reproducibility constraints in our large-$T$ setting. OMGP scales at least cubically with the number of pointwise observations; on our hardware, OMGP$_3$ required $\sim$1 hour for 50{,}000 pointwise observations, making direct comparison impractical without substantial method modification. For GP-CDE, we did not find a publicly available implementation---reproducing the method would require an engineering effort beyond the scope of this work.

We evaluate all methods using three groups of metrics. The first
quantifies how closely the predicted density matches the true
conditional, using divergence measures between predicted and
ground-truth densities computed per response marginal, with additional
joint divergences in the multivariate-response setting. The second
assesses calibration through probability integral transform (PIT)
statistics and empirical coverage at nominal levels. The third measures
pointwise predictive accuracy via held-out log-likelihood and CRPS. For multivariate cases, metrics in the second and third groups are computed marginally for each response then averaged. Because our primary interest is conditional density estimation rather than point prediction, we weight the discussion of results accordingly.

Across all experiments we split inputs 80/20 into training and test sets. For univariate outputs we align components across fits by sorting on the fitted means; for multivariate outputs we use the Hungarian matching method, as described in Section~\ref{sec:fitting}.

In addition to these external comparisons, we ablate a GGMP's two main
structural simplifications. Section~\ref{sec:additive_manufacturing} tests plug-in variances in a low-$N$ regime---the Gaussian approximation underlying each mixture fit is least favorable. Section~\ref{sec:weight_effect} compares shared weights against equal weights and input-dependent weights. Together, both assess whether the restriction justified by Proposition~\ref{prop:universality} costs predictive accuracy in practice.

\subsection{Synthetic Function} 
\label{sec:synth_func}
The synthetic dataset is a distribution-valued regression benchmark on a non-Gaussian one-dimensional input domain $x\in[-3,3]$, with $N=300$ input locations and $T=2000$ samples per location.
The latent trend is
\[
f(x)=0.95\sin(1.05x-0.30)+0.55\sin(2.45x+0.80)-0.38\tanh(1.8x)+0.075x^3.
\]
Conditional densities are generated from a hierarchical mixture
\[
p(y\mid x)=\sum_{b=1}^{3} w_b(x)\sum_{j=1}^{4}\alpha_{b,j}(x)\,\mathcal N\!\big(y\mid \mu_{b,j}(x),\sigma^2_{b,j}(x)\big),
\]
where both $\mu_{b,j}(x)$ and $\sigma^2_{b,j}(x)$ vary smoothly with $x$, and $w_b(x)$ is
controlled by a smooth separation field so that some regions are effectively single-mode
while others are strongly multimodal with overlap and asymmetry. Samples are generated
via inverse-CDF sampling and converted into empirical conditional densities on a fixed
$y$-grid. Figure~\ref{fig:synthetic_function} visualizes this function.
\begin{figure}[t]
 \centering
 \includegraphics[width=\linewidth]{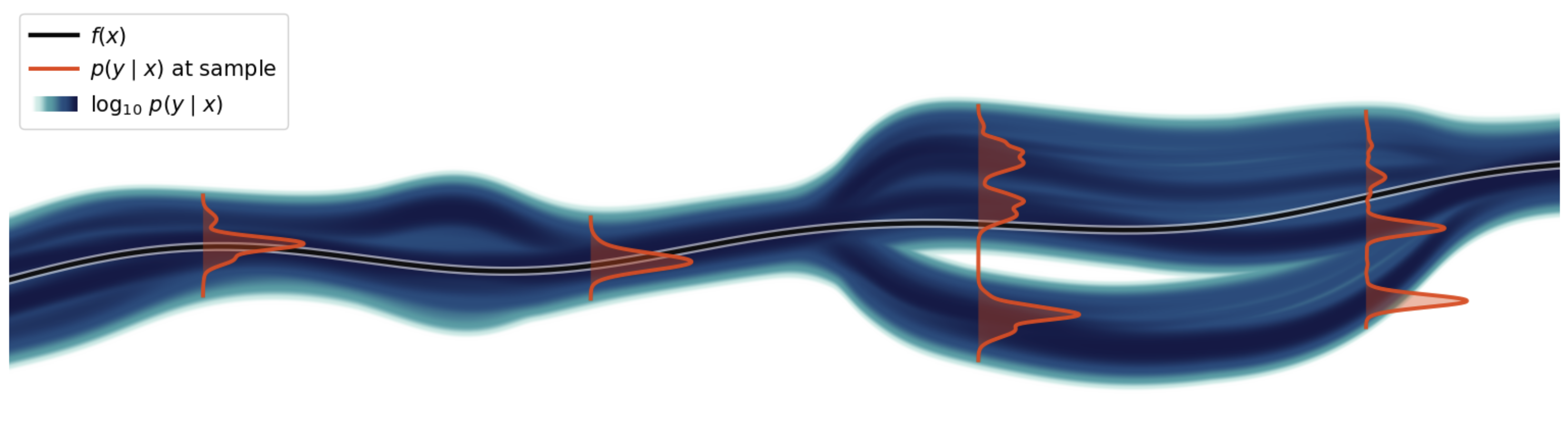}
 \caption{\emph{Synthetic distribution field with fixed-point slices.} The background shows $\log_{10} p(y\mid x)$, generated from a single latent mean function $f(x)$ (black curve), while red vertical slices depict sampled local conditional densities at selected $x$-locations. Despite sharing the same underlying $f(x)$, the local distributions vary strongly in modality, spread, and asymmetry.}
 \label{fig:synthetic_function}
\end{figure}

\begin{table}[h!]
\centering
\small
\begin{tabular}{l|cccc}
\hline
Model
& Bhattacharyya
& Symmetric KL
& Wasserstein-1
& $L^1$ \\
\hline
$\mathrm{GP}_{1}$
& 0.4869 $\pm$ 0.1513
& 4.2005 $\pm$ 1.7752
& 0.7589 $\pm$ 0.2126
& 1.3115 $\pm$ 0.1699 \\
\hline
$\mathrm{MDN}_{1}$
& 0.2278 $\pm$ 0.2124
& 2.9952 $\pm$ 3.0088
& 0.3223 $\pm$ 0.3369
& 0.7633 $\pm$ 0.4602 \\
$\mathrm{MDN}_{3}$
& 0.0169 $\pm$ 0.0049
& 0.1561 $\pm$ 0.0486
& 0.0429 $\pm$ 0.0357
& 0.1989 $\pm$ 0.0586 \\
$\mathrm{MDN}_{5}$
& 0.0163 $\pm$ 0.0048
& 0.1606 $\pm$ 0.0533
& 0.0366 $\pm$ 0.0309
& 0.1936 $\pm$ 0.0522 \\
$\mathrm{MDN}_{10}$
& 0.0156 $\pm$ 0.0038
& 0.1316 $\pm$ 0.0382
& 0.0470 $\pm$ 0.0415
& 0.1925 $\pm$ 0.0462 \\
$\mathrm{MDN}_{25}$
& 0.0179 $\pm$ 0.0028
& 0.1357 $\pm$ 0.0446
& \textbf{0.0296 $\pm$ 0.0244}
& \textbf{0.1839 $\pm$ 0.0282} \\
\hline
$\mathrm{GGMP}_{3}$
& 0.0722 $\pm$ 0.0334
& 0.3322 $\pm$ 0.1727
& 0.1710 $\pm$ 0.0848
& 0.4980 $\pm$ 0.0944 \\
$\mathrm{GGMP}_{5}$
& 0.0602 $\pm$ 0.0342
& 0.2856 $\pm$ 0.2089
& 0.1867 $\pm$ 0.1357
& 0.4374 $\pm$ 0.0966 \\
$\mathrm{GGMP}_{10}$
& 0.0267 $\pm$ 0.0200
& 0.1372 $\pm$ 0.1468
& 0.0752 $\pm$ 0.0449
& 0.2452 $\pm$ 0.0833 \\
$\mathrm{GGMP}_{25}$
& \textbf{0.0149 $\pm$ 0.0105}
& \textbf{0.0744 $\pm$ 0.0732}
& 0.0589 $\pm$ 0.0343
& 0.1875 $\pm$ 0.0591 \\
\hline
\end{tabular}
\caption{Divergence metrics for the synthetic dataset.}
\label{tab:1d_distributional_metrics_meanstd}
\end{table}

Table~\ref{tab:1d_distributional_metrics_meanstd} shows that both a GGMP and an MDN
improve with increasing $K$ and dramatically outperform their
$K=1$ baselines, confirming that unimodal models are inadequate for this
data-generating mechanism. MDNs achieve lower divergence scores at small $K$,
reflecting their ability to learn input-dependent weights and variances
end-to-end. The gap narrows with increasing $K$: at $K=25$, the GGMP leads on
Bhattacharyya and symmetric KL while MDN leads on Wasserstein-1 and $L^1$,
suggesting that GP-based interpolation better captures global density shape
whereas the MDN's flexible parameterization better resolves fine local
transport structure. Sample predictive fits for varying $K$ can be found in
Appendix~\ref{app:synth_pred_fits}.

\begin{table}[h!]
\centering
\small
\begin{tabular}{l|cc|ccccc}
\hline
Model
& Log Score
& CRPS
& PIT Mean
& PIT Std
& Cov 50\%
& Cov 90\%
& Cov 95\% \\
\hline
$\mathrm{GP}_{1}$
& -1.6003 & 1.4150 & 0.4774 & 0.2681 & 0.5648 & 0.8854 & 0.9330 \\
\hline
$\mathrm{MDN}_{1}$
& -0.9420 & 0.5585 & 0.4822 & 0.2779 & 0.5115 & 0.9143 & 0.9570 \\
$\mathrm{MDN}_{3}$
& -0.3572 & 0.4935 & 0.5004 & 0.2290 & 0.6710 & 0.9762 & 0.9890 \\
$\mathrm{MDN}_{5}$
& -0.3505 & 0.4945 & 0.4914 & 0.2286 & 0.6700 & 0.9767 & 0.9884 \\
$\mathrm{MDN}_{10}$
& -0.3458 & 0.4943 & 0.4988 & 0.2300 & 0.6679 & 0.9755 & 0.9886 \\
$\mathrm{MDN}_{25}$
& -0.3415 & 0.4934 & 0.5029 & 0.2289 & 0.6713 & 0.9766 & 0.9892 \\
\hline
$\mathrm{GGMP}_{3}$
& -0.6002 & 1.3912 & 0.5034 & 0.2586 & 0.5633 & 0.9501 & 0.9677 \\
$\mathrm{GGMP}_{5}$
& -0.5582 & 1.3158 & 0.5078 & 0.2645 & 0.5604 & 0.9486 & 0.9709 \\
$\mathrm{GGMP}_{10}$
& -0.4392 & 1.3228 & 0.5020 & 0.2771 & 0.5280 & 0.9249 & 0.9645 \\
$\mathrm{GGMP}_{25}$
& -0.3999 & 1.1950 & 0.5025 & 0.2810 & 0.5237 & 0.9144 & 0.9580 \\
\hline
\end{tabular}
\caption{Predictive scoring and calibration diagnostics on the synthetic dataset.}
\label{tab:synthetic_predictive_scores}
\end{table}

Table~\ref{tab:synthetic_predictive_scores} reveals a notable contrast in
calibration. The GGMP produces well-calibrated PIT statistics across all $K$ (PIT
means within $0.008$ of $0.5$; 90\% and 95\% coverage bracketing nominal
levels), confirming that plug-in variances are sufficiently precise in this
data-rich setting. MDNs with $K \geq 3$ achieve superior log scores and CRPS
but exhibit systematic overcoverage---90\% intervals capture
${\sim}97.6\%$ of observations---indicating overdispersed predictive
densities, a common consequence of learning regularity entirely from data
without an explicit smoothness prior.

\subsection{U.S. Temperature Extremes} 
The training data consist of a large corpus of U.S.\ surface-air temperature observations: roughly $50$ million individual measurements collected at approximately $7{,}000$ weather stations distributed across the United States over about $10$ years \cite{Menne2012}. We represent the data as a time-by-station array of temperatures, with station input $x_n=(\mathrm{lon}_n,\mathrm{lat}_n).$
For each station, we marginalize over time and treat the resulting values as samples from a station-specific marginal distribution. We construct $\hat p_n(y)$ using fixed-bin histograms: for station $n$, bin edges span that station's observed range, with bin width $\Delta y_n=\frac{y_{n,\max}-y_{n,\min}}{B},$ and normalization $\sum_j \hat p_n(y_{n,j})\,\Delta y_{n,j}=1.$ Thus, support is represented on a station-specific grid (histogram bin centers), not KDE bandwidth selection. 

The GGMP operates on per-station distributional summaries, so its cost scales with the number of stations regardless of how many observations underlie each summary.
MDNs train on individual $(x,y)$ pairs; we uniformly subsample $10$
million pairs from the dataset to keep MDN training and evaluation tractable. 

\begin{table}[h!]
\centering
\small
\begin{tabular}{l|cccc}
\hline
Model
& Bhattacharyya
& Symmetric KL
& Wasserstein-1
& $L^1$ \\
\hline
$\mathrm{GP}_{1}$
& 0.1580 $\pm$ 0.1050 & 2.6399 $\pm$ 1.6705 & 2.2896 $\pm$ 0.9970 & 0.6090 $\pm$ 0.1987 \\
\hline
$\mathrm{MDN}_{1}$ & 0.1480 $\pm$ 0.1033 & 2.6022 $\pm$ 1.0942 & 1.7232 $\pm$ 1.0110 & 0.6011 $\pm$ 0.1801\\
$\mathrm{MDN}_{3}$ & 0.1428 $\pm$ 0.1011 & 2.5923 $\pm$ 1.2212 & 1.2998 $\pm$ 1.0232 & 0.5801 $\pm$ 0.2336 \\
$\mathrm{MDN}_{5}$ & 0.1414 $\pm$ 0.1006 & 2.5834 $\pm$ 1.6252 & 1.2121 $\pm$ 0.9547 & 0.5244 $\pm$ 0.2482 \\
$\mathrm{MDN}_{10}$ & 0.1425 $\pm$ 0.0989 & 2.5901 $\pm$ 0.9998 & 1.0829 $\pm$ 1.1029 & 0.5328 $\pm$ 0.2434 \\
$\mathrm{MDN}_{25}$ & \textbf{0.1382 $\pm$ 0.1001} & 2.5602 $\pm$ 1.4353 & \textbf{1.0021 $\pm$ 0.9887} & 0.4932 $\pm$ 0.2223 \\
\hline
$\mathrm{GGMP}_{3}$
& 0.1419 $\pm$ 0.1052 & 2.5760 $\pm$ 1.6838 & 1.3671 $\pm$ 1.0594 & 0.5110 $\pm$ 0.2429 \\
$\mathrm{GGMP}_{5}$
& 0.1379 $\pm$ 0.1038 & 2.5597 $\pm$ 1.6844 & 1.1999 $\pm$ 1.0137 & 0.4928 $\pm$ 0.2463 \\
$\mathrm{GGMP}_{10}$
& 0.1376 $\pm$ 0.1032 & \textbf{2.5594 $\pm$ 1.6823} & 1.2085 $\pm$ 1.0914 & 0.4916 $\pm$ 0.2458 \\
$\mathrm{GGMP}_{25}$
& 0.1375 $\pm$ 0.1025 & 2.5603 $\pm$ 1.6797 & 1.1958 $\pm$ 1.0899 & \textbf{0.4909 $\pm$ 0.2438} \\
\hline
\end{tabular}
\caption{Divergence metrics for the temperature dataset.}
\label{tab:3d_distributional_metrics_meanstd}
\end{table}

Table~\ref{tab:3d_distributional_metrics_meanstd} indicates that increasing
$K$ improves distributional fit primarily from $K=1$ to moderate values, after
which gains largely saturate. The GGMP and MDN achieve comparable divergence
scores at matched $K$, with MDN$_{25}$ attaining the best Bhattacharyya
and Wasserstein-1, while GGMP$_{25}$ posts the best $L^1$
and GGMP$_{10}$ the best symmetric KL. The differences
between the two families are small relative to the standard deviations,
suggesting that with sufficient data both models recover similar distributional
structure; the choice between them is better distinguished by calibration and
uncertainty quantification. Sample fits are previewed in
Figure~\ref{fig:intro_predictive_fits}.

\begin{table}[h!]
\centering
\small
\begin{tabular}{l|cc|ccccc}
\hline
Model
& Log Score
& CRPS
& PIT Mean
& PIT Std
& Cov 50\%
& Cov 90\%
& Cov 95\% \\
\hline
$\mathrm{GP}_{1}$
& -3.7951 & 13.7364 & 0.5134 & 0.3050 & 0.4194 & 0.9065 & 0.9624 \\
\hline
$\mathrm{MDN}_{1}$
& -3.7571 & 7.3935 & 0.5201 & 0.4387 & 0.4223 & 0.5997 & 0.6600 \\
$\mathrm{MDN}_{3}$
& -3.6779 & 6.1454 & 0.5030 & 0.3138 & 0.4868 & 0.7547 & 0.8015 \\
$\mathrm{MDN}_{5}$
& -3.6770 & 6.0630 & 0.4993 & 0.3055 & 0.4889 & 0.8013 & 0.8553 \\
$\mathrm{MDN}_{10}$ & -3.6699 & 5.7824 & 0.5001 & 0.3043 & 0.4989 & 0.8892 & 0.9342 \\
$\mathrm{MDN}_{25}$ & -3.6822 & 5.2364 & 0.4817 & 0.3570 & 0.4518 & 0.8438 & 0.8911 \\
\hline
$\mathrm{GGMP}_{3}$
& -3.7118 & 7.2930 & 0.5040 & 0.2906 & 0.5009 & 0.8887 & 0.9382 \\
$\mathrm{GGMP}_{5}$
& -3.6916 & 6.6387 & 0.5013 & 0.2883 & 0.5028 & 0.8976 & 0.9453 \\
$\mathrm{GGMP}_{10}$
& -3.6892 & 5.8346 & 0.5015 & 0.2886 & 0.4990 & 0.9013 & 0.9484 \\
$\mathrm{GGMP}_{25}$
& -3.6895 & 4.7765 & 0.5018 & 0.2899 & 0.4972 & 0.8987 & 0.9485 \\
\hline
\end{tabular}
\caption{Predictive scoring and calibration diagnostics on the temperature extremes dataset.}
\label{tab:temp_extremes_predictive_scores}
\end{table}

Table~\ref{tab:temp_extremes_predictive_scores} separates the two models more
clearly. The GGMP maintains well-calibrated coverage across all $K$. MDNs are systematically undercovered, most severely at $K=1$ but still noticeably at $K=25$. This gap reflects the lack of principled epistemic uncertainty in the MDN: without a smoothness prior or posterior variance, the predictive intervals are too narrow. The GGMP also achieves the best CRPS at $K=25$, while MDN$_{10}$ posts the best log score, consistent with MDNs concentrating density sharply around modes at
the expense of tail coverage. The $K=1$ GP baseline again shows the largest
calibration departure among GP-based models,
confirming that, here, unmodeled multimodality, rather than the plug-in variance, is the dominant source of miscalibration. Across $K$, GGMP improves CRPS while preserving calibration, whereas MDNs raise log score by overconcentrating mass without epistemic uncertainty, producing overly narrow intervals and persistent undercoverage.

\subsection{Additive Manufacturing} 
\label{sec:additive_manufacturing}
We evaluate the GGMP on a proprietary additive-manufacturing dataset with repeated measurements collected at a finite set of process conditions, serving as a multivariate, multi-task example. To preserve confidentiality, we anonymize all feature and response names and report only their structure: each condition is represented by a 2D input vector $x_n\in\mathbb{R}^2$, and each replicate observation is a 2D output vector $y_{nt}\in\mathbb{R}^2$ (with response dimensions referred to as Axis 1 and Axis 2). The dataset contains $24$ conditions and $600,000$ raw observations. Our objective is distributional: for each condition $x_n$, we model the full conditional output distribution $p_n(y)$. 

\begin{figure}[b]
 \centering
 \includegraphics[width=\linewidth]{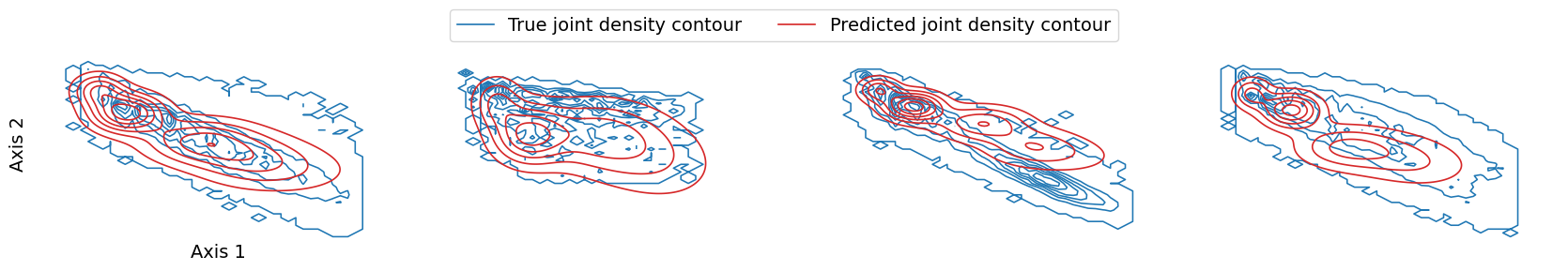}
 \caption{\emph{Joint density reconstruction on held-out conditions (multivariate, \(K=5\)).} Blue contours show the empirical test distribution in the 2D output space (Axis 1–Axis 2), and red contours show the GGMP$_5$ predictive distribution at the same inputs. Each panel corresponds to a different held-out condition. Overall alignment of contour shape and multimodal structure indicates that the GGMP captures the major geometry of the joint distribution.}
 \label{fig:contour}
\end{figure}

\begin{table}[h!]
\centering
\small
\begin{tabular}{l|cc}
\hline
Model
& Energy
& Sliced Wasserstein-1 \\
\hline
$\mathrm{GP}_{1}$ & 0.1764 $\pm$ 0.1182 & 0.2302 $\pm$ 0.0983 \\
\hline
$\mathrm{MDN}_{1}$ & 0.0871 $\pm$ 0.0535 & 0.1980 $\pm$ 0.0555 \\
$\mathrm{MDN}_{3}$ & 0.0820 $\pm$ 0.0739 & 0.1695 $\pm$ 0.0626 \\
$\mathrm{MDN}_{5}$ & 0.0712 $\pm$ 0.0442 & 0.1565 $\pm$ 0.0468 \\
$\mathrm{MDN}_{10}$ & 0.0652 $\pm$ 0.0539 & 0.1414 $\pm$ 0.0579 \\
$\mathrm{MDN}_{25}$ & 0.0652 $\pm$ 0.0586 & 0.1423 $\pm$ 0.0551 \\
\hline
$\mathrm{GGMP}_{3}$ & 0.0659 $\pm$ 0.0314 & 0.1352 $\pm$ 0.0334 \\
$\mathrm{GGMP}_{5}$ & 0.0462 $\pm$ 0.0390 & 0.1071 $\pm$ 0.0454 \\
$\mathrm{GGMP}_{10}$ & \textbf{0.0454 $\pm$ 0.0619} & \textbf{0.0922 $\pm$ 0.0782} \\
$\mathrm{GGMP}_{25}$ & 0.0574 $\pm$ 0.0541 & 0.1219 $\pm$ 0.0619 \\
\hline
\end{tabular}
\caption{Joint metric summary across held-out distributions (mean $\pm$ std).}
\label{tab:joint_metrics_meanstd}
\end{table}

\begin{table}[h!]
\centering
\small
\begin{tabular}{l|cccc}
\hline
Model
& Bhattacharyya
& Symmetric KL
& Wasserstein-1
& $L^1$ \\
\hline
$\mathrm{GP}_{1}$ & 0.1526 $\pm$ 0.0662 & 1.9729 $\pm$ 0.9653 & 0.3436 $\pm$ 0.1769 & 0.6563 $\pm$ 0.1645 \\
\hline
$\mathrm{MDN}_{1}$ & 0.1078 $\pm$ 0.0331 & 2.0924 $\pm$ 0.6608 & 0.1399 $\pm$ 0.1464 & 0.4713 $\pm$ 0.1012 \\
$\mathrm{MDN}_{3}$ & 0.0884 $\pm$ 0.0361 & 1.4679 $\pm$ 0.5218 & 0.1674 $\pm$ 0.1453 & 0.4608 $\pm$ 0.1318 \\
$\mathrm{MDN}_{5}$ & 0.0772 $\pm$ 0.0167 & 1.1918 $\pm$ 0.2770 & 0.1536 $\pm$ 0.1401 & 0.4595 $\pm$ 0.0683 \\
$\mathrm{MDN}_{10}$ & 0.0734 $\pm$ 0.0330 & 1.2620 $\pm$ 0.5892 & 0.1154 $\pm$ 0.1391 & 0.4205 $\pm$ 0.1010 \\
$\mathrm{MDN}_{25}$ & 0.0734 $\pm$ 0.0285 & 1.2057 $\pm$ 0.3635 & 0.1514 $\pm$ 0.1300 & 0.4152 $\pm$ 0.1306 \\
\hline
$\mathrm{GGMP}_{3}$ & 0.0832 $\pm$ 0.0128 & 0.9536 $\pm$ 0.3575 & 0.1907 $\pm$ 0.0612 & 0.4893 $\pm$ 0.0406 \\
$\mathrm{GGMP}_{5}$ & 0.0665 $\pm$ 0.0214 & 0.8474 $\pm$ 0.4050 & 0.1441 $\pm$ 0.0750 & 0.4382 $\pm$ 0.0597 \\
$\mathrm{GGMP}_{10}$ & \textbf{0.0573 $\pm$ 0.0321} & \textbf{0.6255 $\pm$ 0.3896} & \textbf{0.1157 $\pm$ 0.1127} & \textbf{0.4039 $\pm$ 0.1301} \\
$\mathrm{GGMP}_{25}$ & 0.0754 $\pm$ 0.0313 & 0.8691 $\pm$ 0.4403 & 0.1669 $\pm$ 0.0954 & 0.4699 $\pm$ 0.1289 \\
\hline
\end{tabular}
\caption{Marginal metric summary for Axis 1 (mean $\pm$ std).}
\label{tab:marginal_axis1_metrics_meanstd}
\end{table}

\begin{table}[h!]
\centering
\small
\begin{tabular}{l|cccc}
\hline
Model
& Bhattacharyya
& Symmetric KL
& Wasserstein-1
& $L^1$ \\
\hline
$\mathrm{GP}_{1}$ & 0.1751 $\pm$ 0.1213 & 2.0973 $\pm$ 1.5235 & 0.0954 $\pm$ 0.0527 & 0.7393 $\pm$ 0.3263 \\
\hline
$\mathrm{MDN}_{1}$ & 0.3139 $\pm$ 0.1338 & 5.2772 $\pm$ 1.7122 & 0.0579 $\pm$ 0.0537 & 0.9554 $\pm$ 0.2201 \\
$\mathrm{MDN}_{3}$ & 0.2382 $\pm$ 0.1520 & 3.9513 $\pm$ 2.2437 & 0.0682 $\pm$ 0.0623 & 0.8096 $\pm$ 0.2654 \\
$\mathrm{MDN}_{5}$ & 0.2192 $\pm$ 0.1645 & 3.5751 $\pm$ 2.4068 & 0.0708 $\pm$ 0.0437 & 0.7614 $\pm$ 0.2973 \\
$\mathrm{MDN}_{10}$ & 0.2217 $\pm$ 0.1461 & 3.5517 $\pm$ 2.2105 & 0.0881 $\pm$ 0.0366 & 0.7930 $\pm$ 0.2525 \\
$\mathrm{MDN}_{25}$ & 0.2063 $\pm$ 0.1501 & 3.4022 $\pm$ 2.3282 & 0.0695 $\pm$ 0.0488 & 0.7459 $\pm$ 0.2736 \\
\hline
$\mathrm{GGMP}_{3}$ & 0.1543 $\pm$ 0.1610 & 1.5859 $\pm$ 1.8617 & 0.0733 $\pm$ 0.0586 & 0.6687 $\pm$ 0.4414 \\
$\mathrm{GGMP}_{5}$ & 0.1215 $\pm$ 0.1060 & 1.2305 $\pm$ 1.1837 & \textbf{0.0672 $\pm$ 0.0509} & 0.6329 $\pm$ 0.3409 \\
$\mathrm{GGMP}_{10}$ & 0.1133 $\pm$ 0.0964 & \textbf{1.1007 $\pm$ 0.9603} & 0.0673 $\pm$ 0.0548 & \textbf{0.6101 $\pm$ 0.3253} \\
$\mathrm{GGMP}_{25}$ & \textbf{0.1119 $\pm$ 0.0909} & 1.1026 $\pm$ 0.9997 & 0.0694 $\pm$ 0.0552 & 0.6154 $\pm$ 0.3455 \\
\hline
\end{tabular}
\caption{Marginal metric summary for Axis 2 (mean $\pm$ std).}
\label{tab:marginal_axis2_metrics_meanstd}
\end{table}

Tables~\ref{tab:joint_metrics_meanstd}--\ref{tab:marginal_axis2_metrics_meanstd}
show that the GGMP consistently outperforms both the GP baseline and
MDNs across joint and marginal metrics. On joint metrics,
$\mathrm{GGMP}_{10}$ reduces energy distance by $74\%$ relative to
$\mathrm{GP}_1$ and by $30\%$ relative to the best MDN. On both
marginals, $\mathrm{GGMP}_{10}$ achieves the lowest or
near-lowest scores on all divergence measures, with symmetric
KL roughly half that of the best MDN on Axis~1. The gap widens on Axis~2. The MDN's relative weakness here is consistent with the low-$N$ regime: with only $N = 24$ training conditions the network has few input--output pairs from which to learn a smooth mapping, whereas a GGMP's kernel prior provides a stronger inductive bias when training inputs are scarce. Within the GGMP family, performance peaks at $K = 10$ and degrades at $K = 25$, consistent with the finite-sample bias--variance tradeoff discussed in Appendix~\ref{app:K_sensitivity}. Sample fits are shown in Figure~\ref{fig:contour}.

\begin{table}[h!]
\centering
\small
\begin{tabular}{l|cc|ccccc}
\hline
Model
& Log Score
& CRPS
& PIT Mean
& PIT Std
& Cov 50\%
& Cov 90\%
& Cov 95\% \\
\hline
$\mathrm{GP}_{1}$
& -6.6629 & 0.8140 & 0.6810 & 0.2852 & 0.4568 & 0.6856 & 0.7286 \\
\hline
$\mathrm{MDN}_{1}$
& -6.4721 & 0.2088 & 0.5088 & 0.2169 & 0.7114 & 0.9823 & 0.9936 \\
$\mathrm{MDN}_{3}$
& -6.1633 & 0.2071 & 0.5038 & 0.2434 & 0.6349 & 0.9720 & 0.9898 \\
$\mathrm{MDN}_{5}$
& -6.2886 & 0.2053 & 0.5102 & 0.2567 & 0.5860 & 0.9699 & 0.9885 \\
$\mathrm{MDN}_{10}$
& -5.5040 & 0.2024 & 0.4983 & 0.2566 & 0.5857 & 0.9681 & 0.9886 \\
$\mathrm{MDN}_{25}$
& -6.4337 & 0.2013 & 0.5055 & 0.2509 & 0.6040 & 0.9739 & 0.9884 \\
\hline
$\mathrm{GGMP}_{3}$
& -6.6525 & 0.7710 & 0.6918 & 0.2838 & 0.4109 & 0.6991 & 0.7414 \\
$\mathrm{GGMP}_{5}$
& -6.5454 & 0.7030 & 0.6590 & 0.3104 & 0.3728 & 0.7027 & 0.7478 \\
$\mathrm{GGMP}_{10}$
& -6.5591 & 0.4174 & 0.6171 & 0.3139 & 0.4219 & 0.7317 & 0.7781 \\
$\mathrm{GGMP}_{25}$
& -6.6197 & 0.6408 & 0.6070 & 0.3105 & 0.4381 & 0.7665 & 0.7989 \\
\hline
\end{tabular}
\caption{Predictive scoring and calibration diagnostics on the additive manufacturing dataset.}
\label{tab:additive_predictive_scores}
\end{table}

Table~\ref{tab:additive_predictive_scores} shows a clear calibration split. MDNs attain better CRPS and log scores, with PIT means near $0.5$, but they systematically overcover, indicating overly diffuse predictive variances, as in the synthetic study. GGMPs show the opposite pattern: PIT means and 90\% coverage are both below nominal. The issue is not missing GP posterior variance, which is included through $\nu_{*k}$, but the plug-in approximation: test-time within-component spread is replaced by the training average $\bar{s}_k^2$, and uncertainty from local mixture fitting, alignment, and input-dependent spread is not propagated. With only $N=24$ training conditions, this makes held-out predictive mixtures too sharp. Although $T\approx25{,}000$ gives reliable local mixtures at observed conditions, sparse input coverage forces held-out predictions to depend heavily on GP extrapolation. Richer mixtures improve coverage, but it remains below nominal.

\subsection{Effect of weight optimization}
\label{sec:weight_effect}

The shared-weight design adopted in
Section~\ref{sec:training} raises a natural question about how much
weight optimization contributes beyond equal weights $w_k = 1/K$,
and whether input-dependent weights $w(x)$ would further improve
performance. We investigate both questions here. We evaluate the distributional log-likelihood \eqref{eq:ggmp_dens_obj_impl} at equal weights and at the
optimized weights $w^*$ from \eqref{eq:weight_opt} for all three
datasets and all values of $K$. A negligible lift would indicate that alignment alone captures most of the model's expressiveness.

\begin{table}[h]
\centering
\small
\begin{tabular}{ll|rrr}
\hline
Dataset & GGMP$_K$ & $\mathcal{L}_{\mathrm{dens}}(w_{1/K})$
  & $\mathcal{L}_{\mathrm{dens}}(w^*)$
  & $\Delta_{w_{1/K} \to w^*}$ \\
\hline
Synthetic & 3 & $-$381.08 & $-$380.64 & 0.12\% \\
  & 5 & $-$345.32 & $-$345.18 & 0.04\% \\
  & 10 & $-$307.32 & $-$306.12 & 0.39\% \\
  & 25 & $-$301.29 & $-$299.62 & 0.55\% \\
\hline
Temperature & 3 & $-$22,379.2 & $-$22,378.1 & $<$0.01\% \\
  & 5 & $-$22,306.1 & $-$22,304.9 & $<$0.01\% \\
  & 10 & $-$22,322.5 & $-$22,315.0 & 0.03\% \\
  & 25 & $-$22,384.2 & $-$22,376.0 & 0.04\% \\
\hline
Additive & 3 & $-$1,937,212 & $-$1,881,962 & 2.85\% \\
  & 5 & $-$1,966,832 & $-$1,917,670 & 2.50\% \\
  & 10 & $-$1,917,485 & $-$1,878,141 & 2.05\% \\
  & 25 & $-$2,143,543 & $-$2,101,179 & 1.98\% \\
\hline
\end{tabular}
\caption{Effect of weight optimization on the distributional
log-likelihood. $\Delta_{w_{1/K} \to w^*}$ is the relative improvement
$(\mathcal{L}_{\mathrm{opt}} - \mathcal{L}_{\mathrm{equal}}) /
|\mathcal{L}_{\mathrm{equal}}| \times 100$.}
\label{tab:weight_sensitivity}
\end{table}

As shown in Table~\ref{tab:weight_sensitivity}, for the synthetic and temperature datasets, optimizing weights changes the objective by less than $0.6\%$ across all
values of $K$. For the additive-manufacturing dataset, the effect
is more significant, around $2$--$3\%$. This pattern is consistent with the universality guarantee of Proposition~\ref{prop:universality}: when $N$ is large, the component GPs learn accurate mean functions and can absorb components to compensate for uniform weighting, so additional degrees of freedom in $w$ provide little benefit. When $N$ is small, the GP posteriors are less flexible, so weight reallocation
corrects the mixture shape.

Importantly, the concavity of \eqref{eq:weight_opt} means that
shared weight optimization is both cheap and reliable---usually a fraction of the total training time (Appendix~\ref{app:runtimes})---and the global optimum is guaranteed. This makes it a low-risk default that captures
most of the benefit when equal weights are insufficient, without
introducing the additional parameters or potential overfitting of
a richer weight model. 

\subsubsection{Extending to input-dependent weights}
\label{sec:input_dependent_weights}
The shared-weight model constrains all inputs to share a single
simplex vector $w \in \Delta^{K-1}$. When mode prominence varies
across the inputs, $w(x) \in \Delta^{K-1}$ could improve fidelity. 

Since Table~\ref{tab:weight_sensitivity} shows that $\Delta_{w_{1/K} \to w^*}$
is negligible for the synthetic and temperature datasets, we
restrict this comparison to the additive-manufacturing one.

We parameterize input-dependent weights via a softmax-linear model,
\begin{equation}
\label{eq:xdep_weights}
w_k(x) \;=\; \frac{\exp(\beta_k^\top x + b_k)}
 {\sum_{j=1}^{K} \exp(\beta_j^\top x + b_j)},
\end{equation}
with one anchor class fixed at $\beta_K = 0$, $b_K = 0$ for
identifiability. Given the trained component GPs, we optimize
$\{\beta_k, b_k\}_{k=1}^{K-1}$ to maximize the distributional
log-likelihood \eqref{eq:ggmp_dens_obj_impl} with $w_k$ replaced
by $w_k(x_n)$, using L-BFGS initialized at the shared-weight
solution (i.e.\ $\beta_k = 0$, $b_k = \log w_k^*$).
Table~\ref{tab:weight_xdep} reports the results.

\begin{table}[h]
\centering
\small
\begin{tabular}{l|rrrr}
\hline
$K$ & $\mathcal{L}(w{=}1/K)$
 & $\mathcal{L}(w^*)$
 & $\mathcal{L}(w^*(x))$
 & $\Delta_{w^* \to w^*(x)}$ \\
\hline
 3 & $-$1,937,212 & $-$1,881,962 & $-$1,877,461 & 0.24\% \\
 5 & $-$1,966,832 & $-$1,917,670 & $-$1,902,757 & 0.78\% \\
10 & $-$1,917,485 & $-$1,878,141 & $-$1,868,909 & 0.49\% \\
25 & $-$2,143,543 & $-$2,101,179 & $-$2,088,154 & 0.62\% \\
\hline
\end{tabular}
\caption{Distributional log-likelihood on the
additive-manufacturing dataset under equal weights, optimized
shared weights, and optimized input-dependent weights
\eqref{eq:xdep_weights}. $\Delta_{w^* \to w^*(x)}$ is the
relative improvement from shared to input-dependent weights,
$(\mathcal{L}(w^*(x)) - \mathcal{L}(w^*)) /
|\mathcal{L}(w^*)| \times 100$.}
\label{tab:weight_xdep}
\end{table}

Combined with the results in Table~\ref{tab:weight_sensitivity},
this suggests a clear hierarchy: in data-rich regimes with clear
mode prominence, equal weights already suffice; when $N$ is small
or prominence varies across the input domain, weight optimization
is beneficial, but the transition from shared to input-dependent
weights occupies a narrow regime---$N$ must be small enough that
equal weights are inadequate yet large enough to estimate the
additional parameters in \eqref{eq:xdep_weights} reliably.
Because these regimes are difficult to identify a priori and the concave optimization in \eqref{eq:weight_opt} adds negligible
cost (see Appendix~\ref{app:runtimes} for runtimes), we retain shared optimized weights as the default throughout.

\section{Discussion and Conclusion}
\label{sec:discussion}

The experiments in Section~\ref{sec:experiments} illustrate three
regimes. On the synthetic benchmark, where there is ample data, all
distributional metrics improve up to $K = 25$,
confirming that approximation error dominates. On the temperature
dataset, multimodality is moderate and gains saturate beyond
$K \approx 3$--$5$. On the additive-manufacturing dataset,
performance peaks at $K = 10$ and degrades at $K = 25$, reflecting
the finite-sample bias--variance tradeoff analyzed in
Appendix~\ref{app:K_sensitivity}. Across all, GGMP$_{\geq3}$
substantially outperforms GP baselines, and matches or surpasses MDN models. Appendix~\ref{app:runtimes} summarizes computational costs---for comparable $K$, the GGMP is faster than MDN models.

\subsection{Selecting $K$ and handling redundancy}
\label{sec:selecting_K}

The number of mixture components $K$ is the principal structural
hyperparameter. In practice, it can be chosen by held-out
distributional validation or cross-validation
\citep{luo2025asymptotic,tibshirani2001estimating}: fit a GGMP over a grid of candidate values
of $K$, evaluate the distributional log-likelihood on a validation set, and select the maximizing value. Because the component GPs are independent and the weight
optimization is low-dimensional, models for different $K$ can be
fit in parallel, making the overhead of a grid search modest.
Alternatively, information-theoretic criteria applied to the local
mixture fits---such as BIC evaluated on each $\hat{P}_n$---can
provide a per-input mode count, and a robust
quantile across inputs can then be used as a global choice. We describe the failure modes in $K$-specification in Appendix~\ref{app:K_sensitivity}.

Increasing $K$ beyond the intrinsic number of modes does not force
redundant components to receive exactly zero weight. Instead, the
weight optimization typically distributes small weights across clusters of
nearly indistinguishable components, so the mixture behaves as
though it had fewer effective components even without strict
sparsity. This makes moderate over-specification relatively benign in practice,
although it can still reduce the effective sample size per component and thereby
degrade estimation in finite samples, as discussed above. If explicit sparsity is
desired, the weight objective \eqref{eq:ggmp_dens_obj_impl} can be augmented with a
sparsity-inducing regularizer---for example, a negative-entropy
penalty
$\mathcal{L}_{\mathrm{reg}}(w)=\mathcal{L}_{\mathrm{dens}}(h^*, w)+\lambda\sum_{k=1}^{K} w_k\log w_k$,
or a sparse Dirichlet prior $w\sim\mathrm{Dir}(\alpha\mathbf{1}_K)$
with $\alpha<1$.

\subsection{Limitations and future directions}
\label{sec:limitations}
Component alignment relies on sorting (scalar outputs) or
sequential Hungarian matching (multivariate outputs), both of
which are deterministic and greedy. When component tracks cross
frequently or the input domain lacks a natural sequential
structure, these heuristics can introduce label-switching artifacts
in the GP training targets. Global optimal transport across all
inputs simultaneously, or soft probabilistic alignment that
marginalizes over permutations, could mitigate this at additional
cost.

The plug-in approach conditions on locally estimated component means and within-component variances rather than propagating their uncertainty through the predictive density. GP posterior uncertainty in the aligned component means is retained through $\nu_{*k}$, but uncertainty from local mixture fitting, alignment, and test-time within-component variance is omitted. When the number of training inputs $N$ is small, this can lead to overconfident predictions (Section~\ref{sec:additive_manufacturing}). In the data-rich synthetic and temperature settings these effects
are negligible. A hierarchical extension that marginalizes over
posterior uncertainty---rather than conditioning on plug-in means
and variances---would widen predictive intervals where data are
sparse without sacrificing sharpness elsewhere.

Combining the GGMP with scalable GP methods such as inducing-point approximations or structured kernel interpolation would enable application to large-$N$ regimes where $O(KN^3)$ training is prohibitive. The same scalable infrastructure would also permit extension to massive pointwise datasets in settings where each input is associated with a large number of repeated observations $T$, which can be compressed into distribution-valued summaries before GP training.

Several additional extensions would broaden the framework's scope. Classification settings with discrete responses could be addressed by replacing the Gaussian mixture predictive form with categorical or softmax-based alternatives. Spatiotemporal applications would require kernels that model dependence across both space and time. Online updating would enable streaming data settings where retraining from scratch is impractical. Finally, tempered likelihoods \citep{li2026robust} could improve robustness to model misspecification or outliers.

\subsection{Summary}

The GGMP provides a tractable and modular framework for non-Gaussian
process regression that preserves closed-form Gaussian conditioning
within every component while producing  multimodal predictive
densities through mixture combination. Its reliance on standard GP
solvers, parallelizable training, and a principled distributional
objective make it a practical drop-in extension for settings where
conventional GP assumptions are insufficient. The limitations
identified above are addressable, and we expect that resolving them
will further extend the range of problems for which mixture-based
GP modeling is competitive with more complex approximate-inference
alternatives.

\paragraph{Code Availability}
The code is available at \url{https://github.com/vt2211/GGMP}.

\paragraph{Acknowledgments}
This work was supported by The Center for Advanced Mathematics for Energy Research Applications (CAMERA), which is jointly funded by the Advanced Scientific Computing Research (ASCR) and Basic Energy Sciences (BES) within the Department of Energy's Office of Science, under Contract No. DE-AC02-05CH11231. Additional support was provided by the U.S. Department of Energy, Office of Science, Office of Advanced Scientific Computing Research's Applied Mathematics Competitive Portfolios program under the same contract. Early conception and ideation were supported by the Laboratory Directed Research and Development Program of Lawrence Berkeley National Laboratory under U.S. Department of Energy Contract No. DE-AC02-05CH11231.

\paragraph{Author Contributions.}
V.T.: Ideation; Derivation; Implementation; Manuscript writing; Manuscript proofreading; Data curation; Experimentation. 
M.R.: Ideation; Derivation; Manuscript proofreading. 
H.L.: Ideation; Manuscript proofreading.
M.N.: Ideation; Derivation; Implementation; Manuscript writing; Manuscript proofreading; Data curation; Project oversight. 

\newpage

\bibliographystyle{plainnat}
\bibliography{literature}

\newpage

\appendix

\section{Gaussian Process Regression: Background}
\label{app:gp_background}

The GP regression framework models observations as
\begin{equation}
 y(x_i) = f(x_i) + \epsilon(x_i),
 \qquad \forall\, i \in \{1,2,\ldots,N\},
 \label{eq:gp_obs_model}
\end{equation}
where $f(x)$ is an unknown latent function and $\epsilon(x)$ 
is additive Gaussian noise. We write $N := |\mathcal{D}|$ for 
the number of training inputs throughout. Collecting the 
latent function values at the training inputs into 
$\mathbf{f} = [f(x_1), \ldots, f(x_N)]^\top$, the GP prior is 
specified as a multivariate normal distribution,
\begin{equation}
 p(\mathbf{f}) =
 \mathcal{N}(\mathbf{f} \mid \mathbf{m}, \mathbf{K}),
 \label{eq:priorGP}
\end{equation}
where $\mathbf{m} = [m(x_1),\ldots,m(x_N)]^\top$ is the vector 
of prior mean function evaluations and $\mathbf{K}$ is the 
covariance matrix induced by a kernel function $k$, with 
entries $\mathbf{K}_{ij} = k(x_i, x_j; h)$ for hyperparameters 
$h$. Assuming additive Gaussian noise, the likelihood is
\begin{equation}
 p(\mathbf{y}\mid\mathbf{f}) =
 \mathcal{N}(\mathbf{y} \mid \mathbf{f}, \mathbf{V}),
 \label{eq:likelihoodGP}
\end{equation}
where $\mathbf{V}$ is the noise covariance matrix, which may be 
non-i.i.d. The combination of Gaussian prior and likelihood 
permits closed-form marginalization over $\mathbf{f}$, yielding 
the log marginal likelihood
\begin{align}
 \log p(\mathbf{y} \mid h)
 &= \int_{\mathbb{R}^{N}}
 p(\mathbf{y}\mid\mathbf{f})\,p(\mathbf{f})\,\mathrm{d}\mathbf{f}
 \nonumber \\
 &= -\tfrac{1}{2}\log\lvert(2\pi)(\mathbf{K}
 {+}\mathbf{V})\rvert
 - \tfrac{1}{2}
 (\mathbf{y}-\mathbf{m})^{T}
 (\mathbf{K}+\mathbf{V})^{-1}
 (\mathbf{y}-\mathbf{m}),
 \label{eq:likelihood}
\end{align}
which is maximized over $h$ to train the model. For notational 
convenience we absorb $\mathbf{V}$ into $\mathbf{K}$ 
throughout, i.e., $\mathbf{K} \leftarrow \mathbf{K} + 
\mathbf{V}$. After training, the posterior predictive at a 
new input $x^*$ is
\begin{align}
 p(f^*\mid\mathbf{y}) =
 \mathcal{N}\!\Big(
 &\,m(x^*) + \mathbf{k}(x^*)^{T}\mathbf{K}^{-1}
 (\mathbf{y}-\mathbf{m}), 
 \nonumber \\
 &\,k(x^*,x^*) - \mathbf{k}(x^*)^{T}\mathbf{K}^{-1}
 \mathbf{k}(x^*)
 \Big),
 \label{eq:posterior}
\end{align}
where $\mathbf{k}(x^*) := [k(x_1,x^*),\ldots,k(x_N,x^*)]^\top$ 
is the vector of covariances between the training inputs and 
$x^*$. It is the closed-form conditioning in 
\eqref{eq:posterior} that the GGMP retains within every mixture 
component.

\section{Proofs and Theoretical Results}
\label{app:derivations}

\subsection{Proof of Proposition~\ref{prop:universality}}
\label{app:universality}
\begin{proof}

\textbf{Convention.} Throughout we use
$D_{\mathrm{TV}}(p,q)=\frac{1}{2}\int|p(y)-q(y)|\,dy$.

Fix $\epsilon>0$.

\medskip
\noindent\textbf{Step 1 (Uniform continuity in $x$).}
By the uniform tail assumption, choose $R>0$ such that
\[
\sup_{x\in\mathcal{X}}\int_{|y|>R}p^*(y\mid x)\,dy
<\frac{\epsilon}{6}.
\]
On the compact set $\mathcal{X}\times[-R,R]$, joint continuity
implies uniform continuity, so there exists $\eta>0$ such that
$\|x-x'\|<\eta$ implies
$|p^*(y\mid x)-p^*(y\mid x')|<\epsilon/(6R)$ for all
$y\in[-R,R]$. Then for $\|x-x'\|<\eta$,
\begin{align*}
\int|p^*(y\mid x)-p^*(y\mid x')|\,dy
&\leq \int_{-R}^{R}\frac{\epsilon}{6R}\,dy
+\int_{|y|>R}p^*(y\mid x)\,dy
+\int_{|y|>R}p^*(y\mid x')\,dy\\
&<\frac{\epsilon}{3}+\frac{\epsilon}{6}+\frac{\epsilon}{6}
=\frac{2\epsilon}{3}.
\end{align*}
Therefore
$D_{\mathrm{TV}}(p^*(\cdot\mid x),\,p^*(\cdot\mid x'))<\epsilon/3$.
Set $\rho=\eta$ and extract a finite $\rho$-cover
$\{x_1^*,\ldots,x_J^*\}$ of $\mathcal{X}$.

\noindent\textbf{Step 2 (Gaussian mixture approximation at each
center).}

\emph{Convolution approximation.}
For any $L^1$ density $f$, the convolution
$f*\mathcal{N}(0,\sigma^2)$ converges to $f$ in $L^1$ as
$\sigma^2\to 0$ (the standard approximate-identity property of
the Gaussian kernel). Since there are finitely many centers,
choose $\sigma^2>0$ such that
\[
\max_{1\leq j\leq J}\,
D_{\mathrm{TV}}\!\big(p^*(\cdot\mid x_j^*),\;
p^*(\cdot\mid x_j^*)*\mathcal{N}(0,\sigma^2)\big)
<\frac{\epsilon}{6}.
\]
Note that
$p^*(\cdot\mid x_j^*)*\mathcal{N}(0,\sigma^2)
=\int p^*(t\mid x_j^*)\,\mathcal{N}(\cdot\mid t,\sigma^2)\,dt$,
i.e., a Gaussian location mixture with mixing measure
$p^*(t\mid x_j^*)\,dt$.

\emph{Quantile discretization.}
For each center $x_j^*$, let $F_{x_j^*}$ denote the CDF of
$p^*(\cdot\mid x_j^*)$. Define the quantile points
$t_n^{(j)}=F_{x_j^*}^{-1}\!\big((n-\tfrac{1}{2})/N\big)$
for $n=1,\ldots,N$, and the discrete mixing measure
$\nu_j^N=\frac{1}{N}\sum_{n=1}^{N}\delta_{t_n^{(j)}}$.
As $N\to\infty$, $\nu_j^N$ converges weakly to
$p^*(\cdot\mid x_j^*)\,dt$.

\emph{Weak convergence to $L^1$ via Scheff\'e.}
Define the mixture densities
\[
f_j^N(y)
=\int\mathcal{N}(y\mid t,\sigma^2)\,d\nu_j^N(t)
=\frac{1}{N}\sum_{n=1}^{N}\mathcal{N}(y\mid t_n^{(j)},\sigma^2)
\]
and
$f_j(y)=\int\mathcal{N}(y\mid t,\sigma^2)\,p^*(t\mid x_j^*)\,dt$.
Since $t\mapsto\mathcal{N}(y\mid t,\sigma^2)$ is bounded and
continuous for each fixed $y$, weak convergence
$\nu_j^N\Rightarrow p^*(\cdot\mid x_j^*)\,dt$ gives
$f_j^N(y)\to f_j(y)$ for every $y$. Both $f_j^N$ and $f_j$ are
nonnegative and integrate to~1, so by Scheff\'e's lemma,
$\int|f_j^N(y)-f_j(y)|\,dy\to 0$.

Choose $N_0$ large enough that for all $j=1,\ldots,J$,
\[
D_{\mathrm{TV}}\!\left(
p^*(\cdot\mid x_j^*)*\mathcal{N}(0,\sigma^2),\;
\frac{1}{N_0}\sum_{n=1}^{N_0}
\mathcal{N}(\cdot\mid t_n^{(j)},\sigma^2)\right)
<\frac{\epsilon}{6}.
\]
By the triangle inequality, defining
$q_j(\cdot)=\frac{1}{N_0}\sum_{n=1}^{N_0}
\mathcal{N}(\cdot\mid t_n^{(j)},\sigma^2)$,
\[
D_{\mathrm{TV}}\!\big(p^*(\cdot\mid x_j^*),\,q_j\big)
<\frac{\epsilon}{6}+\frac{\epsilon}{6}
=\frac{\epsilon}{3}.
\]

\noindent\textbf{Step 3 (Partition of unity interpolation).}
Let $\{\phi_j\}_{j=1}^J$ be a continuous partition of unity on
$\mathcal{X}$ subordinate to the open cover
$\{B(x_j^*,\rho)\}_{j=1}^J$. Define
\[
\tilde{q}(\cdot\mid x)
=\sum_{j=1}^{J}\phi_j(x)\,q_j(\cdot).
\]
By joint convexity of total variation distance, for any
$x\in\mathcal{X}$,
\[
D_{\mathrm{TV}}\!\big(p^*(\cdot\mid x),\,
\tilde{q}(\cdot\mid x)\big)
\leq\sum_{j}\phi_j(x)\,
D_{\mathrm{TV}}\!\big(p^*(\cdot\mid x),\,q_j\big).
\]
For each $j$ with $\phi_j(x)>0$, we have
$\|x-x_j^*\|<\rho$, so by Steps~1 and~2,
\[
D_{\mathrm{TV}}\!\big(p^*(\cdot\mid x),\,q_j\big)
\leq D_{\mathrm{TV}}\!\big(p^*(\cdot\mid x),\,
p^*(\cdot\mid x_j^*)\big)
+D_{\mathrm{TV}}\!\big(p^*(\cdot\mid x_j^*),\,q_j\big)
<\frac{\epsilon}{3}+\frac{\epsilon}{3}
=\frac{2\epsilon}{3}.
\]
Therefore
$\sup_{x\in\mathcal{X}}
D_{\mathrm{TV}}(p^*(\cdot\mid x),\,
\tilde{q}(\cdot\mid x))<2\epsilon/3$.

Expanding, $\tilde{q}(\cdot\mid x)$ is a Gaussian location
mixture with shared variance $\sigma^2$ and mixing measure
\[
\nu_x
=\sum_{j=1}^{J}\sum_{n=1}^{N_0}
\frac{\phi_j(x)}{N_0}\,\delta_{t_n^{(j)}}
\]
supported on the finite set
$S=\{t_n^{(j)}:j=1,\ldots,J,\;n=1,\ldots,N_0\}$.

\noindent\textbf{Step 4 (Conversion to uniform weights).}
Order the distinct elements of $S$ as $s_1<\cdots<s_P$. The
mixing measure $\nu_x$ assigns continuous weight
$\alpha_p(x)\geq 0$ to $s_p$, with
$\sum_{p=1}^P\alpha_p(x)=1$. Define the cumulative weights
$C_0(x)=0$ and
$C_i(x)=\sum_{p=1}^{i}\alpha_p(x)$ for $i=1,\ldots,P$.

For $m=1,\ldots,M$, define
\[
\mu_m(x)
=\sum_{i=1}^{P}s_i\,w_i^{(m)}(x),
\qquad
w_i^{(m)}(x)
=M\!\left(
\min\!\big(C_i(x),\,m/M\big)
-\max\!\big(C_{i-1}(x),\,(m{-}1)/M\big)
\right)^{\!+}\!.
\]

\emph{Continuity.}
Each $w_i^{(m)}$ is a composition of continuous functions
($C_i$, $\min$, $\max$, $(\cdot)^+$), hence continuous. So each
$\mu_m:\mathcal{X}\to[s_1,s_P]$ is continuous.

\emph{Weights sum to~1.}
For fixed $m$ and $x$, the quantity $w_i^{(m)}(x)/M$ equals the
length of the intersection of the interval
$(C_{i-1}(x),C_i(x)]$ with $[(m{-}1)/M,\,m/M]$. Summing over
$i$ gives total length $1/M$, so
$\sum_i w_i^{(m)}(x)=1$.

\emph{Marginal recovery.}
$\frac{1}{M}\sum_{m=1}^{M}w_i^{(m)}(x)
=\mathrm{Leb}\!\big((C_{i-1}(x),C_i(x)]\big)
=\alpha_i(x)$.

\emph{Wasserstein bound.}
Define the coupling
$\gamma_x
=\sum_{m,i}\frac{w_i^{(m)}(x)}{M}\,
\delta_{(s_i,\,\mu_m(x))}$
between $\nu_x$ and
$\hat{\nu}_x=\frac{1}{M}\sum_{m=1}^{M}\delta_{\mu_m(x)}$.
Call component $m$ ``good'' if its window
$[(m{-}1)/M,\,m/M]$ lies entirely within some interval
$(C_{i-1}(x),C_i(x)]$; then $w_i^{(m)}(x)=1$ and
$\mu_m(x)=s_i$, contributing zero transport cost. Each
interior boundary $C_i(x)$ for $i=1,\ldots,P-1$ is straddled
by at most one window, so there are at most $P-1$ ``bad''
components, each contributing transport cost at most
$\mathrm{diam}(S)/M$. Hence
\[
W_1(\nu_x,\,\hat{\nu}_x)
\leq\frac{(P-1)\,\mathrm{diam}(S)}{M}.
\]

\emph{TV bound via Gaussian Lipschitz property.}
We use the following standard estimate. For univariate Gaussians
with equal variance,
\[
D_{\mathrm{TV}}\!\big(\mathcal{N}(t_1,\sigma^2),\,
\mathcal{N}(t_2,\sigma^2)\big)
=2\Phi\!\!\left(\frac{|t_1-t_2|}{2\sigma}\right)-1
\leq\frac{|t_1-t_2|}{\sigma\sqrt{2\pi}},
\]
where the inequality follows from
$\Phi(a)-\frac{1}{2}\leq\frac{a}{\sqrt{2\pi}}$ for $a\geq 0$.
In other words, the map
$t\mapsto\mathcal{N}(\cdot\mid t,\sigma^2)$ is Lipschitz from
$(\mathbb{R},|\cdot|)$ to
$(L^1(\mathbb{R}),\,D_{\mathrm{TV}})$ with Lipschitz constant
$L=\frac{1}{\sigma\sqrt{2\pi}}$.

Now, for any coupling $\pi$ of $\nu_x$ and $\hat{\nu}_x$,
\begin{align*}
D_{\mathrm{TV}}\!\left(
\int\mathcal{N}(\cdot\mid t,\sigma^2)\,d\nu_x(t),\;
\int\mathcal{N}(\cdot\mid s,\sigma^2)\,d\hat{\nu}_x(s)
\right)
&\leq\int D_{\mathrm{TV}}\!\big(
\mathcal{N}(t,\sigma^2),\,
\mathcal{N}(s,\sigma^2)\big)\,d\pi(t,s)\\
&\leq\frac{1}{\sigma\sqrt{2\pi}}
\int|t-s|\,d\pi(t,s).
\end{align*}
Taking the infimum over all couplings
(Kantorovich--Rubinstein duality),
\[
D_{\mathrm{TV}}\!\big(\tilde{q}(\cdot\mid x),\,
\hat{q}(\cdot\mid x)\big)
\leq\frac{W_1(\nu_x,\,\hat{\nu}_x)}{\sigma\sqrt{2\pi}}
\leq\frac{(P-1)\,\mathrm{diam}(S)}{M\,\sigma\sqrt{2\pi}}.
\]
Choose $M$ large enough that
\[
\frac{(P-1)\,\mathrm{diam}(S)}{M\,\sigma\sqrt{2\pi}}
<\frac{\epsilon}{3}.
\]

\noindent\textbf{Step 5 (Final bound).}
Set $K=M$. The functions
$\mu_1,\ldots,\mu_K:\mathcal{X}\to\mathbb{R}$ are continuous,
the variance $\sigma^2$ is shared, and the weights are uniform
($1/K$ each). For all $x\in\mathcal{X}$,
\[
D_{\mathrm{TV}}\!\big(p^*(\cdot\mid x),\,
\hat{q}(\cdot\mid x)\big)
\leq
\underbrace{D_{\mathrm{TV}}\!\big(p^*(\cdot\mid x),\,
\tilde{q}(\cdot\mid x)\big)}_{<\,2\epsilon/3}
+\underbrace{D_{\mathrm{TV}}\!\big(\tilde{q}(\cdot\mid x),\,
\hat{q}(\cdot\mid x)\big)}_{<\,\epsilon/3}
<\epsilon.
\qedhere
\]
\end{proof}

\subsection{Proof of Lemma~\ref{distribution_MLE}}
\label{app:proof_prop1}

Fix $n\in\{1,\dots,N\}$. By definition of the KL divergence,
\[
D_{\mathrm{KL}}\!\big(p_n \,\|\, q_\theta(\cdot\mid x_n)\big)
:=\int_{\mathcal Y} p_n(y)\log\frac{p_n(y)}{q_\theta(y\mid x_n)}\,d\mu(y).
\]
Expanding the log-ratio and splitting the integral gives
\[
D_{\mathrm{KL}}\!\big(p_n \,\|\, q_\theta(\cdot\mid x_n)\big)
=\int_{\mathcal Y} p_n(y)\log p_n(y)\,d\mu(y)
-\int_{\mathcal Y} p_n(y)\log q_\theta(y\mid x_n)\,d\mu(y).
\]
Rearranging and using $\int p_n\log p_n\,d\mu=-H(p_n)$,
\[
\int_{\mathcal Y} p_n(y)\log q_\theta(y\mid x_n)\,d\mu(y)
=
- H(p_n)
-
D_{\mathrm{KL}}\!\big(p_n \,\|\, q_\theta(\cdot\mid x_n)\big).
\]
Summing over $n=1,\dots,N$ yields
\[
\mathcal L(\theta)
=
-\sum_{n=1}^N H(p_n)
-\sum_{n=1}^N D_{\mathrm{KL}}\!\big(p_n \,\|\, q_\theta(\cdot\mid x_n)\big).
\]
Since $C:=-\sum_{n=1}^N H(p_n)$ is constant in $\theta$,
\[
\arg\max_\theta \mathcal L(\theta)
=
\arg\min_\theta \sum_{n=1}^N D_{\mathrm{KL}}\!\big(p_n \,\|\, q_\theta(\cdot\mid x_n)\big).
\]

\section{Model Extensions and Implementation Details}
\subsection{Sensitivity to the number of components, $K$}
\label{app:K_sensitivity}

\textbf{Under-specification ($K < K_{\mathrm{true}}$)}.
Suppose the true conditional density at input $x_n$ has
$K_{\mathrm{true}}$ well-separated modes but we fit only
$K < K_{\mathrm{true}}$ components. When a single fitted component
absorbs samples from two distinct modes with means $\mu_a$ and
$\mu_b$ and respective mixing fractions $\alpha$ and
$(1-\alpha)$, the fitted component mean converges to
$\hat{m}_{n\kappa} \to \alpha\mu_a + (1-\alpha)\mu_b$, which lies
between the two modes. By the law of total variance, the fitted
within-component variance satisfies $s^2_{n\kappa} \;\ge\; \alpha(1-\alpha)(\mu_a - \mu_b)^2$, growing quadratically in the mode separation. The GP trained on
these targets via \eqref{eq:gp_mll} therefore learns an
inter-modal mean function, and the inflated variance
$s^2_{n\kappa}$ enters the predictive density
\eqref{eq:q_nk_density_impl} as additional spread, producing a
component that places mass between the true modes rather than on
them. In the limiting case $K=1$, the single component reduces to
the marginal mean and variance of $p_n$, and the GGMP degenerates
to a standard heteroscedastic GP with no  multimodal structure.

\textbf{Over-specification ($K > K_{\mathrm{true}}$).}
Suppose the true conditional has $K_{\mathrm{true}}$ modes but we
fit $K > K_{\mathrm{true}}$ components. For each true mode, the
local mixture fit assigns multiple components that converge to
nearly identical parameters. If $J \ge 2$ fitted components share
the same mean $\mu$ and variance $\sigma^2$, their contribution
to the predictive mixture \eqref{eq:ggmp_station_density_k_impl}
satisfies $ \sum_{j=1}^{J} w_j\,\mathcal{N}(y \mid \mu, \sigma^2) \;=\; (\sum_{j=1}^{J} w_j)\,\mathcal{N}(y \mid \mu, \sigma^2)$,
which is a single Gaussian regardless of how the weights are
distributed among the $J$ components. The mixture is therefore no
more expressive than a $K_{\mathrm{true}}$-component model.
Moreover, the distributional log-likelihood
\eqref{eq:ggmp_dens_obj_impl} depends on $w_j$ and $w_\ell$ only
through their sum $w_j + w_\ell$ whenever $q_{nj} = q_{n\ell}$
for all $n$, so the objective is constant along the direction
$(w_j, w_\ell) \mapsto (w_j + \epsilon,\, w_\ell - \epsilon)$.
The weight optimization \eqref{eq:weight_opt} therefore becomes
ill-conditioned, with the Hessian of
$\mathcal{L}_{\mathrm{dens}}$ acquiring near-zero eigenvalues
along such redistribution directions, though the problem remains
concave---any solution along the flat direction yields the same
predictive density. 

This, however, is true when we have enough data to ensure the convergence of each component. In practice, performance reflects a bias--variance tradeoff governed by the
\emph{effective sample size per component}, $N_k^{\mathrm{eff}}=\sum_{n,t} r_{ntk}$, where $r_{ntk}$ denotes the EM responsibility of component $k$ for observation $t$ at input $x_n$. As $K$ grows, many
$N_k^{\mathrm{eff}}$ decrease, making local component estimation, alignment,
and GP fitting noisier. Consequently, very large $K$ can underperform despite higher expressivity. 

This is the regime observed in our experiments: within an intermediate $K$-range, increasing $K$ gives the best held-out distributional accuracy. We demonstrate both failure modes empirically in Section~\ref{sec:experiments}, visualize such distributions in Appendix~\ref{app:synth_pred_fits}, and discuss practical strategies for selecting $K$ in Section~\ref{sec:discussion}.

\subsection{Extension of Section~\ref{sec:fitting} to multivariate outputs}
\label{app:multivariate}

The derivation above assumes scalar outputs $y \in \mathbb{R}$. We now state the extension to $p$-dimensional outputs $y \in \mathbb{R}^p$. At each input $x_n$, the local Gaussian mixture fit (Section~\ref{sec:fitting}) now produces $p$-dimensional component parameters
$$\big\{(\hat{\boldsymbol{m}}_{n\kappa},\, \mathbf{S}_{n\kappa})\big\}_{\kappa=1}^{K}, \qquad \hat{\boldsymbol{m}}_{n\kappa} \in \mathbb{R}^p, \quad \mathbf{S}_{n\kappa} \in \mathbb{R}^{p \times p},$$
where $\hat{\boldsymbol{m}}_{n\kappa}$ is the fitted component mean vector and $\mathbf{S}_{n\kappa}$ is the within-component covariance matrix. After alignment (using the Hungarian procedure \eqref{eq:hungarian_alignment} with the multivariate $\mathcal{W}_2^2$ cost \eqref{eq:W2_cost}, since sorting is not available without a canonical ordering on $\mathbb{R}^p$), we relabel to obtain $(\hat{\boldsymbol{m}}_{nk}, \mathbf{S}_{nk})$ for $k = 1, \ldots, K$.
For each component $k$ and each output dimension $j \in \{1, \ldots, p\}$, we define a scalar GP training set $\mathcal{D}_{k}^{(j)} := \big\{\big(x_n,\; [\hat{\boldsymbol{m}}_{nk}]_j,\; [\mathbf{S}_{nk}]_{jj}\big)\big\}_{n=1}^{N},$ where $[\hat{\boldsymbol{m}}_{nk}]_j$ is the $j$-th coordinate of the aligned component mean and $[\mathbf{S}_{nk}]_{jj}$ is the corresponding diagonal entry of the within-component covariance, used as the heteroscedastic noise variance. We train an independent GP on each $\mathcal{D}_k^{(j)}$, yielding $Kp$ scalar GP models in total. The posterior predictive for the $j$-th coordinate of the $k$-th component mean at input $x$ is
$[g_k(x)]_j \mid (h_k^{(j)}, \mathcal{D}_k^{(j)}) \sim \mathcal{N}\!\big(\mu_{k}^{(j)}(x),\, \nu_{k}^{(j)}(x)\big),$
where $\mu_k^{(j)}$ and $\nu_k^{(j)}$ are the standard GP predictive mean and variance from the $j$-th coordinate model of component $k$.
Since the $p$ coordinate GPs within each component are trained independently, their joint posterior is a product of independent Gaussians. Marginalizing the latent component mean $\boldsymbol{g}_k(x) \in \mathbb{R}^p$ against a diagonal Gaussian posterior and a diagonal within-component likelihood covariance generalizes \eqref{eq:q_nk_density_impl} to
\begin{equation}
q_{*k}(y^*; h_k^*) = \mathcal{N}\Big(y^* \Big| \boldsymbol{\mu}_{*k}, \boldsymbol{\Sigma}_{*k}\Big),
\label{eq:q_nk_multivariate}
\end{equation}
where
$$\boldsymbol{\mu}_{*k} := \big(\mu_k^{(1)}(x_*),\, \ldots,\, \mu_k^{(p)}(x_*)\big)^\top \in \mathbb{R}^p,$$
$$\boldsymbol{\Sigma}_{*k} := \mathrm{diag}\!\Big(\nu_k^{(1)}(x_*) + \bar{s}^2_{k,1},\;\; \ldots,\;\; \nu_k^{(p)}(x_*) + \bar{s}^2_{k,p}\Big) \in \mathbb{R}^{p \times p},$$
and $\bar{s}^2_{k,j} := \frac{1}{N}\sum_{n=1}^{N} [\mathbf{S}_{nk}]_{jj}$ is the training-set average of the within-component variance for component $k$ along output dimension $j$. As in the scalar case (Section~\ref{sec:prediction}), the averaged variance $\bar{s}^2_{k,j}$ could be replaced by an input-dependent estimate when the within-component spread varies substantially across the input domain. The full GGMP predictive density at a new input $x_*$ is then
\begin{equation}
q(y \mid x_*;\, h^*, w^*) = \sum_{k=1}^{K} w_k^*\, \mathcal{N}\!\Big(y \,\Big|\, \boldsymbol{\mu}_{*k},\; \boldsymbol{\Sigma}_{*k}\Big).
\label{eq:ggmp_pred_multivariate}
\end{equation}
This is a $K$-component Gaussian mixture in $\mathbb{R}^p$ with diagonal component covariances, and can be evaluated, sampled, and differentiated in closed form. The total number of GPs is $Kp$, each of size $N$, so the training cost is $O(KpN^3)$ and per-input prediction cost is $O(KpN^2)$.
Two remarks are in order. First, the diagonal structure of $\boldsymbol{\Sigma}_{*k}$ arises because coordinate GPs are trained independently; it does not assume that the output dimensions are uncorrelated under the true data-generating process. Cross-output correlation is captured implicitly through the mixture structure: even though each component has a diagonal covariance, the mixture \eqref{eq:ggmp_pred_multivariate} can represent arbitrary correlations between output dimensions, since a Gaussian mixture with diagonal components is dense in the space of continuous densities on $\mathbb{R}^p$ as $K \to \infty$. For finite $K$, the approximation quality depends on whether the true cross-output dependence structure is adequately captured by the component means and weights. In settings where strong within-component cross-output correlation is present and $K$ is small, one could instead train a multi-output GP per component to obtain a full (non-diagonal) $\boldsymbol{\Sigma}_{*k}$, at the cost of additional modeling and computational complexity.
Second, the weight optimization \eqref{eq:weight_opt} and the distributional log-likelihood \eqref{eq:ggmp_dens_obj_impl} carry over without modification: the scalar density evaluations $q_{nk}(y; h_k)$ are simply replaced by their multivariate counterparts $\mathcal{N}(y \mid \boldsymbol{\mu}_{nk}, \boldsymbol{\Sigma}_{nk})$ throughout, and concavity in $w$ is preserved since the argument (log of a nonnegative affine function of $w$) is dimension-independent.

\subsection{Computational Cost}
\label{app:runtimes}

All experiments were run on a single CPU node of the Perlmutter supercomputer at NERSC (AMD EPYC 7763, 128 cores, 512 GB RAM); no GPU acceleration was used. The GGMP pipeline parallelizes naturally at the GP training stage: each of the $K$ component GPs is independent, so we assign one worker per component and train all $K$ models concurrently. The remaining stages run serially on the same node. MDN baselines were trained using PyTorch with data-parallel batching across available CPU cores.

The reported GP training times are wall-clock makespans, bounded by the slowest component rather than the sum of its components. In practice, the makespan can increase with $K$ because different components may require different optimization effort, raising the probability of at least one slow GP. Across the synthetic and temperature experiments, GP training dominates, consistent with the $O(KpN^3)$ complexity analysis in Section~\ref{sec:training}. In the manufacturing experiment, where $N$ is small but $T \approx 25{,}000$, mixture fitting and alignment scale strongly with $K$ and $T$ and becomes the primary cost at larger $K$; weight optimization is also non-negligible. Prediction remains negligible in all settings.

\begin{table}[H]
\centering
\small
\begin{tabular}{c|cccc|c||c}
\hline
$K$ & GMM + Align & GP Training & Weight Opt & Prediction & GGMP Total & MDN Total \\
\hline
1 & $<\!1$ & 39 & 0 & 1 & $<\!41$ & 1,318 \\
3 & 9 & 64 & 3 & 1 & 77 & 1,241 \\
5 & 21 & 79 & 9 & 1 & 110 & 1,212 \\
10 & 50 & 132 & 9 & 1 & 192 & 1,178 \\
25 & 144 & 242 & 12 & 1 & 399 & 1,198 \\
\hline
\end{tabular}
\caption{Wall-clock runtimes (seconds) for the synthetic function experiment
($N=300$, $N*T=600{,}000$, $p=1$).}
\label{tab:runtime_synthetic}
\end{table}

\begin{table}[H]
\centering
\small
\begin{tabular}{c|cccc|c||c}
\hline
$K$ & GMM + Align & GP Training & Weight Opt & Prediction & GGMP Total & MDN Total \\
\hline
1 & $<\!1$ & 593 & 0 & 57 & $<\!651$ & 26,613\\
3 & 177 & 4,842 & 28 & 58 & 5,105 & 26,452\\
5 & 329 & 7,150 & 108 & 63 & 7,650 & 27,130\\
10 & 717 & 13,317 & 206 & 70 & 14,310 & 26,828\\
25 & 1,482 & 20,935 & 1,059 & 82 & 23,558 & 27,164\\
\hline
\end{tabular}
\caption{Wall-clock runtimes (seconds) for the U.S.\ temperature extremes experiment
($N \approx 7{,}000$, $N*T_{\text{GGMP}} \approx 50,000,000$, $N*T_{\text{MDN}} = 10,000,000$, $p=1$).}
\label{tab:runtime_temperature}
\end{table}

\begin{table}[H]
\centering
\small
\begin{tabular}{c|cc|cc|c|c||c}
\hline
$K$ & GMM + Align & GP Training
 & \multicolumn{2}{c|}{Weight Opt}
 & Prediction & GGMP Total & MDN Total \\
 &  &
 & $w^*$ & $w^*(x)$
 &  & ($w^*$) &  \\
\hline
1 & $<\!1$ & 3 & 0 & 0 & $<\!1$ & $<\!5$ & 1,009 \\
3 & 15 & 8 & 2 & 19 & 3 & 28 & 1,018 \\
5 & 24 & 11 & 7 & 204 & 5 & 47 & 997 \\
10 & 42 & 18 & 41 & 528 & 6 & 107 & 1,035 \\
25 & 256 & 26 & 123 & 929 & 7 & 412 & 991 \\
\hline
\end{tabular}
\caption{Wall-clock runtimes (seconds) for the additive manufacturing experiment
($N=24$, $N*T\approx600{,}000$, $p=2$). Weight optimization is
reported for shared weights ($w^*$) and input-dependent weights
($w^*(x)$); the GGMP total uses $w^*$.}
\label{tab:runtime_manufacturing}
\end{table}

\subsection{MDN Architecture and Training Details}
\label{app:mdn_details}

All MDN baselines use a fully connected network with two hidden
layers of 64 units each and $\tanh$ activations. The output layer
produces $K$ component means, $K$ variances (parameterized via
softplus), and $K$ mixture weights (parameterized via softmax).
Training minimizes the negative log-likelihood of individual
$(x,y)$ pairs using Adam with a learning rate of $10^{-3}$, a
batch size of 512, and 500 epochs. All runs use a fixed random seed of 42.
The same architecture and hyperparameters are used across all
experiments and all values of
$K \in \{1, 3, 5, 10, 25\}$; no per-$K$ tuning is performed.

\subsection{Sample Fitting for the Synthetic Function}
\label{app:synth_pred_fits}

\begin{figure}[h!]
 \centering
 \includegraphics[width=\linewidth]{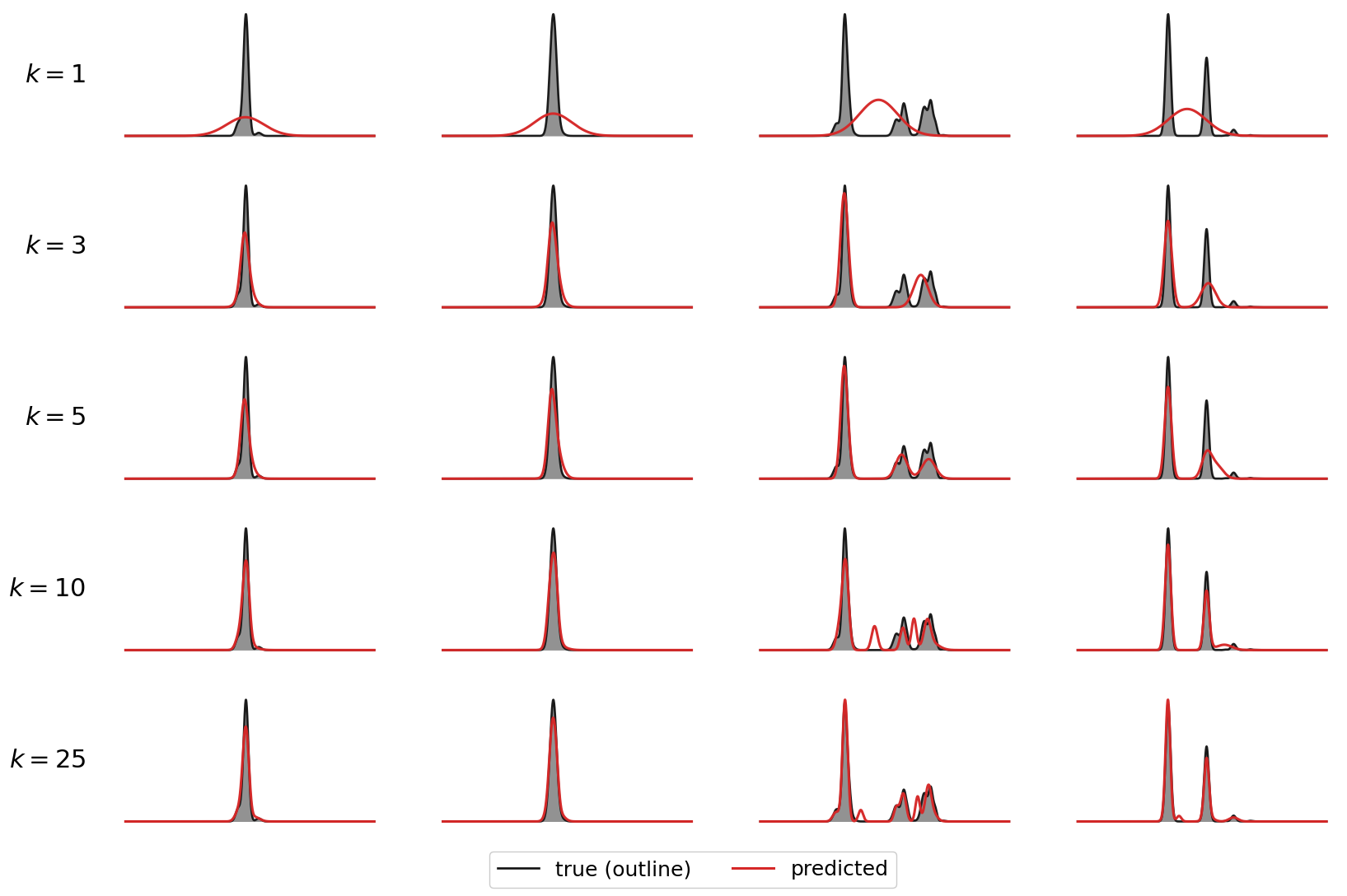}
 \caption{\emph{Held-out distribution reconstruction across mixture complexity for the synthetic function.} Each row shows a different number of mixture components; each column is one of four approximately equally spaced held-out inputs. In every panel, the ground-truth is shown in gray, and the model prediction is red. Increasing $K$ improves local shape fidelity—especially multimodal structure—while low $K$ underfits fine-scale features. When $K < K_{\text{true}}$, we witness an averaging phenomenon; and when $K > K_{\text{true}}$, we observe a distribution that looks like it could have come from a mixture with fewer components, indicating that components are likely stacked.}
 \label{fig:synth_1d_fits}
\end{figure}

\end{document}